\documentclass[sigconf,screen,nonacm]{acmart}
\usepackage{booktabs}
\usepackage{multirow}
\usepackage{colortbl}
\usepackage{xcolor}
\usepackage{array}
\usepackage{makecell}
\usepackage{adjustbox}
\definecolor{headerblue}{RGB}{70,175,185}
\definecolor{lightcyan}{RGB}{200,234,236}
\definecolor{lightgray}{RGB}{220,220,220}
\usepackage{graphicx}
\usepackage{subcaption} 
\usepackage{algorithm}
\usepackage{algpseudocode}

\usepackage{booktabs}
\usepackage{multirow}
\usepackage{amsthm}

\theoremstyle{remark}
\newtheorem{remark}{Remark}
\AtBeginDocument{%
  }

\makeatletter
\newcommand{\maketitlesupp}{%
  \begingroup
    \let\saved@abstract\@abstract
    \let\saved@keywords\@keywords
    \let\saved@translatedabstracts\@translatedabstracts
    \let\saved@translatedkeywords\@translatedkeywords
    \let\saved@mkabstract\@mkabstract

    \def\@abstract{}%
    \def\@keywords{}%
    \def\@translatedabstracts{}%
    \def\@translatedkeywords{}%
    \def\@mkabstract{}%

    \maketitle
  \endgroup
}
\makeatother

\begin{document}

\title{SpikeMLLM: Spike-based Multimodal Large Language Models via Modality-Specific Temporal Scales and Temporal Compression}

\author{Han Xu\texorpdfstring{$^{1,2,3,\dagger}$}{} \quad 
Zhiyong Qin\texorpdfstring{$^{1,2,\dagger}$}{} \quad 
Di Shang\texorpdfstring{$^{1,2}$}{} \quad 
Jiahong Zhang\texorpdfstring{$^{1,2}$}{} \quad 
Xuerui Qiu\texorpdfstring{$^{1,2,4}$}{} \\
Bo Lei\texorpdfstring{$^{3}$}{} \quad 
Tiejun Huang\texorpdfstring{$^{3,5}$}{} \quad 
Bo Xu\texorpdfstring{$^{1,2*}$}{} \quad 
Guoqi Li\texorpdfstring{$^{1,2,6,7*}$}{}}

\affiliation{%
  \institution{
    {\small $^1$Institute of Automation, Chinese Academy of Sciences \quad
    $^2$University of Chinese Academy of Sciences}\\
    {\small $^3$Beijing Academy of Artificial Intelligence \quad
    $^4$Zhongguancun Academy \quad
    $^5$Peking University}\\
    {\small $^6$Key Laboratory of Brain Cognition and Brain-inspired Intelligence Technology}\\
    {\small $^7$Spiking Intelligence Lab, Tianqiao \& Chrissy Chen Institute}\\[0.5em]
    {\small $^\dagger$Equal contribution \quad $^*$Corresponding authors: \texttt{xubo@ia.ac.cn}, \texttt{guoqi.li@ia.ac.cn}}
  }
  \country{}
}

\renewcommand{\shortauthors}{Xu et al.}

\begin{abstract}
Multimodal Large Language Models (MLLMs) have achieved remarkable progress but incur substantial computational overhead and energy consumption during inference, limiting deployment in resource-constrained environments. Spiking Neural Networks (SNNs), with their sparse event-driven computation, offer inherent energy efficiency advantages on neuromorphic hardware, yet extending them to MLLMs faces two key challenges: heterogeneous modalities make uniform spike encoding insufficient, and high-resolution image inputs amplify timestep unfolding overhead. We propose \textbf{SpikeMLLM}, the first spike-based framework for MLLMs, which unifies existing ANN quantization methods in the spiking representation space and incorporates \textbf{Modality-Specific Temporal Scales (MSTS)} guided by Modality Evolution Discrepancy (MED) and \textbf{Temporally Compressed LIF (TC-LIF)} for timestep compression from $T=L-1$ to $T=\log_2(L)-1$. Experiments on four representative MLLMs across diverse multimodal benchmarks show that SpikeMLLM maintains near-lossless performance under aggressive timestep compression ($T_v/T_t=3/4$), with average gaps of only $0.72\%$ and $1.19\%$ relative to the FP16 baseline on InternVL2-8B and Qwen2VL-72B. We further develop a dedicated RTL accelerator tailored to the spike-driven datapath, observing $9.06\times$ higher throughput and $25.8\times$ better power efficiency relative to an FP16 GPU baseline under a deployment-oriented co-design setting, suggesting the promise of algorithm--hardware co-design for efficient multimodal intelligence.
\end{abstract}

\keywords{Spiking Neural Networks, Multimodal Large Language Models, Spike-driven, Efficient Inference}

\maketitle
\pagestyle{plain}

\section{Introduction}
\enlargethispage{\baselineskip}

Multimodal Large Language Models (MLLMs) have achieved remarkable progress in multimodal understanding and reasoning, emerging as a core component of modern AI systems 
\cite{yin2024survey,qwen2vl,internvl2,minicpmv,qwenvl}. However, their inference typically incurs substantial computational overhead and energy consumption, posing significant challenges in resource-constrained environments and limiting their potential as intelligent infrastructure.

Spiking Neural Networks (SNNs) transmit information via discrete spikes, transforming dense multiplications into a sparse event-driven paradigm \cite{xu2025neurormorphic,maass1997networks,fang2023spikingjelly,roy2019towards}, which demonstrates significant energy efficiency advantages when combined with neuromorphic hardware \cite{frenkel2023bottom,schuman2022opportunities,kudithipudi2025neuromorphic,pei2019towards}. Moreover, biological neural systems are known to exhibit temporal processing mechanisms, where different sensory modalities often operate at distinct temporal scales \cite{tavanaei2019deep,pfeiffer2018deep}. SNNs inherently possess temporal dynamics, making them mechanistically well-suited for multimodal temporal processing, thus providing a promising paradigm for energy-efficient multimodal intelligent systems.
In recent years, SNNs have expanded from early visual recognition to natural language processing and large-scale language models, with researchers exploring paradigms including direct training \cite{wu2018spatio,eshraghian2023training}, ANN-to-SNN conversion \cite{hu2023spiking,rathi2020enabling}, and integer-to-spike unfolding \cite{yao2024spikev2, lei2025spike2former, yao2025scaling, xu2025neurormorphic} — the latter gaining increasing attention for its compatibility with both sequential and parallel temporal patterns. In the direction of language modeling, recent efforts have begun to scale SNNs to large language models (LLMs) \cite{xu2025neurormorphic,xing2025spikellm}, demonstrating promising performance in language tasks. 
However, existing efforts on scaling SNNs toward LLMs remain largely confined to unimodal text scenarios, and SNNs have yet to be introduced into  MLLMs, where multimodal inputs introduce new challenges and open questions.

Bringing SNNs into  MLLMs introduces two key challenges: \textbf{(i)} Visual and language modalities differ significantly in information density, activation distribution, and computational characteristics. A uniform spike encoding strategy fails to accommodate their distinct requirements — struggling to meet the precision demands of the language modality while introducing unnecessary redundancy for the visual modality. \textbf{(ii)} In the integer-to-spike unfolding paradigm, discrete activations with $L$ quantization levels require $T = L - 1$ timesteps \cite{yao2024spikev2, xu2025neurormorphic}. In MLLMs, high-resolution image inputs generate substantially more tokens than unimodal text-based LLMs, further amplifying temporal unfolding overhead and posing a significant efficiency bottleneck for multimodal inference.

To address these challenges, we propose \textbf{SpikeMLLM}, the first spike-based framework for MLLMs, which unifies existing ANN quantization methods within the spiking representation space. With\-in this framework, we further develop two key mechanisms.
For challenge~\textbf{(i)}, we propose \textbf{Modality-Specific Temporal Scales (MSTS)}, which characterizes cross-layer representational evolution discrepancy across modalities via Modality Evolution Discrepancy (MED), and adaptively allocates differentiated timesteps accordingly --- assigning more timesteps to the more dynamically evolving text modality while reducing timesteps for the spatially redundant visual modality, thereby improving overall performance without increasing effective inference overhead. For challenge~\textbf{(ii)}, we propose \textbf{Temporally Compressed LIF (TC-LIF)}, which 
compresses timesteps from $T=L-1$ to $T=\log_2(L)-1$ via temporal weighted firing while preserving spike-driven sparse addition, substantially reducing spiking inference overhead without sacrificing representational capacity. Experiments on four mainstream MLLMs and scalability validation on Qwen2VL-72B \cite{qwen2vl} demonstrate that SpikeMLLM maintains near-lossless performance relative to the FP16 baseline under aggressive timestep compression (i.e., $T_v/T_t=3/4$ for visual/text modalities), with average performance gaps of only 0.72\% and 1.19\% on InternVL2-8B \cite{internvl2} and Qwen2VL-72B, respectively. 
In addition, to study whether these algorithmic simplifications can translate into practical system benefits, we further develop a dedicated RTL accelerator tailored to the induced spike-driven datapath. Under this deployment-oriented co-design setting, we observe favorable throughput and power characteristics relative to an FP16 GPU reference, suggesting the promise of algorithm–hardware co-design for efficient multimodal intelligence.
The main contributions of this paper are summarized as follows:

\begin{itemize}
    \item We introduce SNNs into MLLMs for the first time, proposing SpikeMLLM, a spike-based framework that unifies existing ANN quantization methods within the spiking representation space.
    \item We propose MSTS, which improves multimodal model performance through modality- and layer-wise differentiated timestep allocation without increasing effective timesteps.
    \item We propose TC-LIF, which compresses the timesteps of integer-to-spike unfolding from $T=L-1$ to $T=\log_2(L)-1$, effectively alleviating temporal unfolding overhead in spiking multimodal inference.
    \item Extensive experiments on four mainstream MLLMs across diverse multimodal reasoning tasks, with scalability further demonstrated on Qwen2VL-72B, validate the effectiveness, compatibility, and scalability of SpikeMLLM under aggressive timestep compression.
    \item We validate the energy efficiency advantages of SpikeMLLM through a dedicated RTL accelerator under hardware-software co-design, providing insights for next-generation neuromorphic chip design.
\end{itemize}

\section{Preliminaries}

\subsection{Spiking Neuron}
SNNs transmit information via discrete spikes in the temporal dimension. The Leaky Integrate-and-Fire (LIF) model is widely adopted for its balance between efficiency and biological plausibility~\cite{maass1997networks,li2023brain}:
\begin{equation}
\mathbf{U}^l[t] = \lambda\mathbf{U}^l[t-1] + \mathbf{X}^l[t] - V_{th} \cdot \mathbf{S}^l[t-1]
\end{equation}
\begin{equation}
\mathbf{S}^l[t] = \Theta(\mathbf{U}^l[t] - V_{th}) \label{eq:lif_spike}
\end{equation}
where $\lambda$ is the leakage coefficient, $V_{th}$ is the firing threshold, and $\Theta(\cdot)$ is the Heaviside step function. At each timestep $t$, the neuron accumulates $\mathbf{X}^l[t]$ into membrane potential $\mathbf{U}^l[t]$ and emits spike $\mathbf{S}^l[t]$ if the threshold is exceeded. The standard non-polar mode has $\mathbf{S}^l[t] \in \{0,1\}$, where 1 denotes an excitatory spike and 0 the resting state. The polar firing mode~\cite{guo2024ternary,xing2024spikelm} extends this to $\mathbf{S}^l[t] \in \{-1,0\}$ or $\{0,1\}$ via a dual-threshold mechanism: an excitatory spike ($+1$) is fired when the membrane potential exceeds $V_{th}$, an inhibitory spike ($-1$) when it falls below $-V_{th}$, and 0 otherwise. This enables richer encoding per timestep while preserving event-driven sparse addition.

\subsection{Integer-to-Spike LIF}
However, both standard and polar LIF neurons produce discrete binary spike outputs, which limits representational capacity and becomes a critical performance bottleneck as task complexity and network scale increase. Prior works allow SNN neurons to emit integer spikes ~\cite{Hu2023FastSNNFS,xing2024spikelm}, but such methods violate the spike-driven nature. To this end, the integer-to-spike paradigm~\cite{yao2024spikev2, xu2025neurormorphic} quantizes membrane potentials into integer-valued spike counts, equivalently unfolded into multi-timestep binary spike sequences during inference, preserving event-driven sparsity while enhancing representational capacity. Eq.~(\ref{eq:lif_spike}) can be written as:
\begin{equation}
\mathbf{S}^l = \lfloor\mathrm{clip}(\mathbf{U}^l, 0, T)\rceil = \sum_{t=1}^{T} \hat{\mathbf{S}}^l[t], \quad \hat{\mathbf{S}}^l[t] \in \{0,1\}
\end{equation}
\begin{equation}
\max \mathbf{S}^l = T, \quad T = L - 1 \label{eq:integer_spike}
\end{equation}
where $\mathrm{clip}(\mathbf{U}^l, 0, T)$ truncates the membrane potential to $[0, T]$, and $\lfloor\cdot\rceil$ denotes rounding. For $L$ quantization levels, the maximum integer count is $L-1$. The next-layer input is:
\begin{equation}
\mathbf{X}^{l+1} = \mathbf{W}^{l+1}\sum_{t=1}^{T}\hat{\mathbf{S}}^l[t]
\end{equation}
Since $\hat{\mathbf{S}}^l[t] \in \{0, 1\}$, this reduces to sparse addition, replacing dense matrix multiplication. The timestep unfolding supports both sequential and parallel spike firing modes~\cite{yao2025scaling}, enabling flexible inference. Concretely, an integer spike count $\mathbf{S}^l \in \{0,\ldots,L-1\}$ is unfolded into $T=L-1$ binary timesteps, so $T$ directly determines the per-layer computational cost; in parallel mode, this manifests as increased arithmetic cost rather than increased latency, up to the available parallelism.

\subsection{From ANN Quantization to Spike Encoding}
Quantization compresses deep neural networks by reducing the numerical precision of weights and activations. Depending on whether retraining is required, it falls into two categories: Post-Training Quantization (PTQ)~\cite{frantar2023gptq,ashkboos2024quarot,xiao2023smoothquant} directly quantizes a pre-trained model using a small calibration dataset without backpropagation; Quantization-Aware Training (QAT)~\cite{liu2023llmqat,jacob2018quantization} incorporates quantization noise during training. As introduced above, the integer-to-spike paradigm maps membrane potentials to bounded integer spike counts $\mathbf{S}^l \in \{0, 1, \ldots, L-1\}$, which shares a structurally consistent discrete form with quantized ANN activations — both are representations over a finite discrete set:
\begin{equation}
\mathbf{S}^l \in \{0, 1, \ldots, L-1\} \leftrightarrow 
Q(\mathbf{U}^l) \in \{0, \Delta, \ldots, (L-1)\Delta\}
\end{equation}
This structural consistency enables existing ANN quantization methods to be consistently instantiated within the spiking representation space. This bridge transfers quantized representations into a spike-driven inference paradigm, thereby preserving the event-driven property of SNNs for neuromorphic deployment while narrowing the SNN--ANN performance gap.

\begin{remark}
For an $A$-bit activation quantizer, let $L=2^A$ denote the number of quantization levels. Under standard integer-to-spike unfolding with $T=L-1$, the binary spike sequence satisfies $\sum_{t=1}^{T}\hat{\mathbf{S}}^l[t]=\mathbf{S}^l$ by construction, yielding the same linear-layer output as the original quantized activation, i.e., $\mathbf{X}^{l+1}=\mathbf{W}^{l+1}\mathbf{S}^l=\mathbf{W}^{l+1}\sum_{t=1}^{T}\hat{\mathbf{S}}^l[t]$. Therefore, the $T=L-1$ configuration serves as the spike-form reference corresponding to $A$-bit activation quantization. The goal of MSTS and TC-LIF is not to reproduce this reference, but to preserve its performance  under aggressive timestep compression. A complete proof is provided in Supplementary Section~\ref{app:reference_equivalence}.
\end{remark}

\section{Method}
\subsection{Overview}

With continuous advances in SNNs and extensive borrowing from ANN methodologies, SNNs have progressively expanded from simple classification to complex sequence and language modeling. Extending SNNs to Multimodal Large Language Models (MLLMs) represents a promising next step.

Existing integer-to-spike methods \cite{yao2024spikev2, xu2025neurormorphic} rely on integer-valued quantization aware training, incurring prohibitive costs at large scale. Although PTQ methods are well-validated in ANNs, their application in spike-based MLLMs remains largely unexplored. We propose \textbf{SpikeMLLM}, a spike-based framework for MLLMs that enables existing ANN quantization methods to be consistently modeled within the spiking representation.

However, two key challenges remain in multimodal settings. First, tokens from different modalities differ significantly in information density and cross-layer evolution, making uniform timestep allocation insufficient. Second, existing integer-to-spike unfolding requires $T=L-1$ timesteps for $L$ discrete levels; in MLLMs, larger token counts and high representational precision requirements further amplify this overhead. To address these, we propose \textbf{Modality-Specific Temporal Scales (MSTS)}, which adaptively allocates timesteps across modalities, and \textbf{Temporally Compressed LIF (TC-LIF)}, which compresses timesteps from $T=L-1$ to $T=\log_2(L)-1$, substantially reducing spike-driven inference overhead.

Different modalities exhibit distinct representational structures and information densities: visual tokens typically contain higher spatial redundancy and show smoother cross-layer evolution, where\-as text tokens, as discrete semantic units, often undergo larger cross-layer changes and are more sensitive to coarse temporal discretization. This is also consistent with biological perceptual systems, where different modalities operate under heterogeneous temporal integration scales~\cite{pfeiffer2018deep,tavanaei2019deep}. Since SNNs naturally encode information through timesteps, we explore whether timestep allocation can be differentiated across modalities.

\begin{figure}[t]
  \centering
  \includegraphics[width=\linewidth]{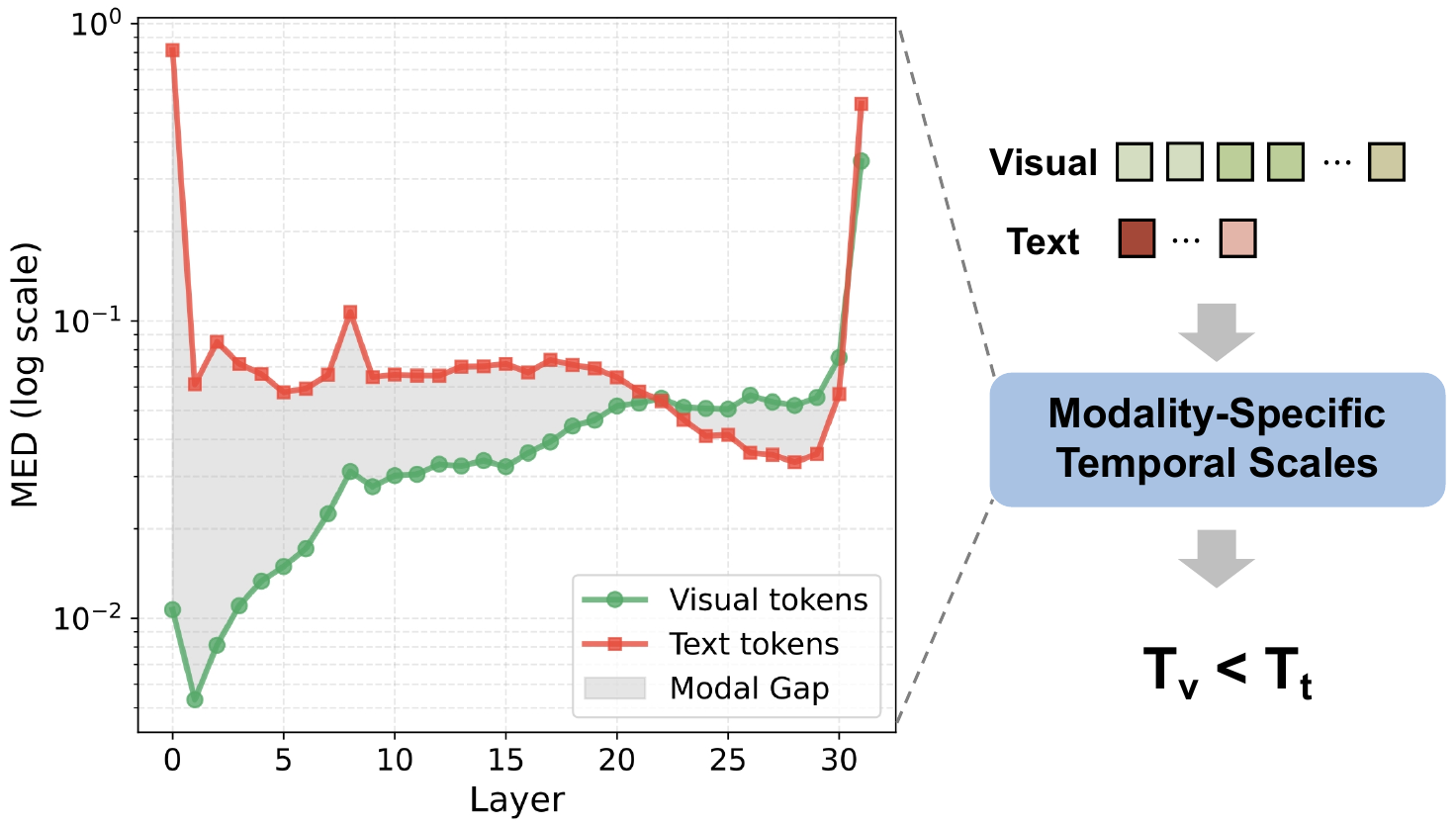}
  \caption{Cross-layer MED discrepancy between visual and text tokens in MLLMs, visualized on Qwen-VL-Chat-9.6B \cite{qwenvl}. Text tokens exhibit higher MED than visual tokens on average, motivating differentiated timestep allocation across modalities in MSTS.}
  \label{fig:med}
\end{figure}

\subsection{Modality-Specific Temporal Scales}
\label{sec:msts}
\paragraph{\textbf{Modality Evolution Discrepancy (MED)}}
In a quantized network, the Transformer updates representations layer by layer through residual connections:
\begin{equation}
h^l = h^{l-1} + Q(F(h^{l-1}))
\end{equation}
where $F(\cdot)$ denotes the Transformer sub-layer and $Q(\cdot)$ is the quantization function. Tokens from different modalities often exhibit different degrees of cross-layer representational change, providing an empirical signal of redundancy and evolution strength. 
Cross-layer cosine similarity has been shown effective for characterizing representational redundancy in Transformers~\cite{he2024matters}. For modality $m \in \{\text{image}, \text{text}\}$, we compute the average cross-layer cosine similarity at layer $l$:
\begin{equation}
\text{Sim}^l_m = \frac{1}{|N_m|}\sum_{i \in N_m}\frac{h^{l-1}_{m,i} \cdot h^l_{m,i}}{\|h^{l-1}_{m,i}\|\|h^l_{m,i}\|}
\end{equation}
where $N_m$ denotes the token set of modality $m$. 

We define the Modality Evolution Discrepancy (MED) as:
\begin{equation}
\text{MED}^l_m = 1 - \text{Sim}^l_m
\end{equation}

A larger $\mathrm{MED}_m^l$ indicates a more significant cross-layer representation update at this layer. For high-MED layers, the error introduced by coarse temporal discretization is more likely to interfere with subsequent representation evolution and propagate through the network, thereby having a larger impact on overall model performance. 
Therefore, modalities or layers with larger MED generally benefit from finer timestep allocation, whereas smaller $\text{MED}^l_m$ suggests higher redundancy and fewer timesteps may suffice. Fig.~\ref{fig:med} illustrates the cross-layer MED discrepancy between visual and text tokens. During prefill, visual tokens substantially outnumber text tokens, so the effective timestep cost is often dominated by the visual modality.

\paragraph{\textbf{MED-Guided Timestep Allocation}}
Based on MED, we adopt a two-level timestep allocation strategy. At the global modality level, we assign differentiated base timesteps according to the overall MED statistics of each modality:
\begin{equation}
T_m =
\begin{cases}
T_t, & \text{if } m = \text{text} \\
T_v, & \text{if } m = \text{visual}
\end{cases},
\quad T_t > T_v
\end{equation}
where $T_t$ and $T_v$ are chosen such that the more dynamically evolving text modality receives more timesteps while the visual modality uses fewer. At the residual-layer level, given a target average timestep budget $\bar{T}_m$, layers of modality $m$ are ranked by MED in descending order, with the top $k_m$ layers assigned higher timestep $T_m$ and the remaining layers assigned lower timestep $T'_m$, satisfying:
\begin{equation}
\frac{k_m T_m + (L_m - k_m) T'_m}{L_m} = \bar{T}_m,
\quad
k_m = \frac{\bar{T}_m - T'_m}{T_m - T'_m} \cdot L_m
\end{equation}
where $L_m$ is the total number of layers for modality $m$. In this way, the limited high-timestep budget is prioritized for layers with larger MED, improving representational fidelity while controlling overall timestep overhead.

\begin{figure}[t]
  \centering
  \includegraphics[width=0.9\linewidth]{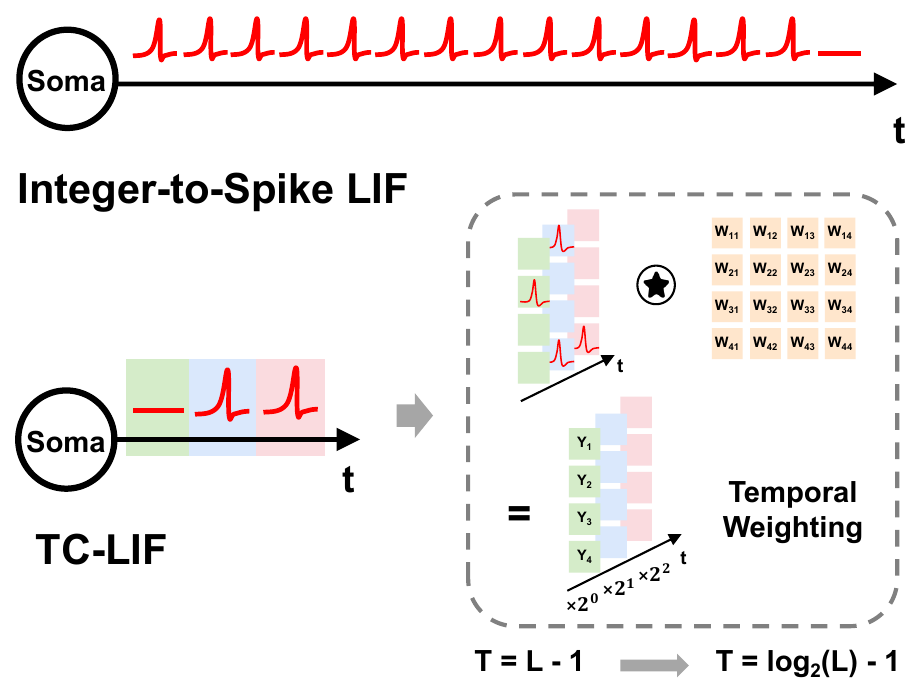}
    \caption{Comparison between Integer-to-Spike LIF and TC-LIF. 
    Integer-to-Spike LIF requires $T = L - 1$ timesteps with uniform 
    spike sequences, whereas TC-LIF reduces timesteps to $T = \log_2(L) - 1$ via temporal weighted integration.}
    \label{fig:tc}
\end{figure}

\subsection{Temporally Compressed LIF}
Conventional integer-to-spike unfolding requires $T=L-1$ timesteps (Eq.~(\ref{eq:integer_spike})) to represent $L$ discrete levels, e.g., $T=15$ for INT4 ($L=16$), leading to substantial temporal overhead in MLLMs. We therefore propose Temporally Compressed LIF (TC-LIF), a spike-preserving temporal reparameterization of integer-to-spike unfolding that reduces the required timesteps from $T=L-1$ to $T=\log_2(L)-1$ while retaining spike-driven sparse accumulation.

\paragraph{\textbf{Spiking Quantization with Symmetric Level Dropping.}}
For modality $m$, we set the quantization scale as $s_m=\frac{\max|\mathbf{U}^{l,m}|}{L_m/2-1}$ and quantize the membrane potential into
\begin{equation}
\mathbf{S}^{l,m} = Q\!\left(\mathbf{U}^{l,m},-L_m/2+1,L_m/2-1\right) = \sum_{t_m=1}^{T_m}2^{t_m-1}\hat{\mathbf{S}}^{l,m}[t_m],
\end{equation}
where $\hat{\mathbf{S}}^{l,m}[t_m]$ denotes a polar binary spike train whose sign is carried by spike polarity. Hence, the temporal code only needs to represent the magnitude. Without level dropping, the magnitude range would be $[0,L_m/2]$, whose maximum $L_m/2=2^{b-1}$ still requires $b$ bits for $L_m=2^b$. By dropping the unreachable level $-L_m/2$, the range becomes $[-L_m/2+1,L_m/2-1]$, so the maximum magnitude is $L_m/2-1=2^{b-1}-1$, which is exactly representable in $b-1$ bits. Therefore, $T_m=b-1=\log_2(L_m)-1$; a formal proof of representational equivalence is provided in Supplementary Section~\ref{app:tclif_equivalence}. This level dropping introduces no extra precision loss, since activating $-L_m/2$ would require a magnitude exceeding $\max |\mathbf{U}^{l,m}|$ under the chosen scaling.

\paragraph{\textbf{Temporal Weighted Integration.}}
Specifically, TC-LIF assigns temporal weight $2^{t_m-1}$ to the spike at timestep $t_m$, such that the interaction with the weight matrix is still realized through spike-driven sparse accumulation rather than dense integer multiplication:
\begin{equation}
\mathbf{X}^{l+1,m}
=\sum_{t_m=1}^{T_m}2^{t_m-1}
\left(\mathbf{W}^{l+1}\hat{\mathbf{S}}^{l,m}[t_m]\right).
\end{equation}
The outputs of different modalities are then concatenated:
\begin{equation}
\mathbf{X}^{l+1}
=\mathrm{Concat}_m\!\left(\mathbf{X}^{l+1,m}\right).
\end{equation}
In this way, TC-LIF reduces the timestep complexity from $L-1$ to $\log_2(L)-1$; each timestep remains a genuine spike-driven sparse addition: when $\hat{S}[t]=0$, no spike is emitted and $\mathbf{W}\hat{S}[t]$ is entirely bypassed, a sparsity that is inherent and pronounced in spiking networks. For INT4 ($L=16$), TC-LIF uses only $T=3$, achieving $5\times$ compression over standard unfolding ($T=15$). The complete procedure is given in Supplementary Algorithm~\ref{alg:tclif}, and a non-polar variant is provided in Supplementary Section~\ref{app:nonpolar}.

\begin{table*}[t]
\centering
\renewcommand{\arraystretch}{1.1}
\setlength{\tabcolsep}{4pt}
\caption{Performance comparison of SpikeMLLM under different configurations across multiple MLLMs and benchmarks. Time Step denotes $T_v/T_t$, the timesteps allocated to visual and text tokens, respectively.}
\label{tab:main_results}
\begin{adjustbox}{width=0.78\linewidth} 
\begin{tabular}{l l c c c c c c c}
\toprule
{{Model}} &
{{Method}} &
{{Spiking}} &
{\makecell{{Time}\\{Step}}} &
{{OCRBench}} &
{{MME}} &
{{TextVQA}} &
{{DocVQA}} &
{{ScienceQA}} \\
\midrule

\multirow{7}{*}{Qwen2VL-7B}
& \cellcolor{lightgray} FP16            & \cellcolor{lightgray} $\times$ & \cellcolor{lightgray} N/A & \cellcolor{lightgray} 835 & \cellcolor{lightgray} 2285 & \cellcolor{lightgray} 84.11 & \cellcolor{lightgray} 93.95 & \cellcolor{lightgray} 85.67 \\
& SpikeMLLM(RTN)                        & $\checkmark$ & 255 & 682  & 1904 & 74.48 & 75.04 & 80.71 \\
& SpikeMLLM(GPTQ)                       & $\checkmark$ & 255 & 760  & 2032 & 80.90 & 89.71 & 83.29 \\
& SpikeMLLM(QuaRot)                     & $\checkmark$ & 7   & 165  & 795  & 3.49  & 3.34  & 21.84 \\
& \cellcolor{lightcyan} SpikeMLLM(QuaRot+MSTS+TC-LIF) & \cellcolor{lightcyan} $\checkmark$ & \cellcolor{lightcyan} 2/3 & \cellcolor{lightcyan} 717 & \cellcolor{lightcyan} 1984 & \cellcolor{lightcyan} 78.28 & \cellcolor{lightcyan} 84.20 & \cellcolor{lightcyan} 80.37 \\
& SpikeMLLM(QuaRot)                     & $\checkmark$ & 15   & 792  & 2196 & 82.64 & 92.50 & 83.54 \\
& \cellcolor{lightcyan} SpikeMLLM(QuaRot+MSTS+TC-LIF) & \cellcolor{lightcyan} $\checkmark$ & \cellcolor{lightcyan} 3/4 & \cellcolor{lightcyan} \textbf{823} & \cellcolor{lightcyan} \textbf{2236} & \cellcolor{lightcyan} \textbf{83.46} & \cellcolor{lightcyan} \textbf{93.31} & \cellcolor{lightcyan} \textbf{85.57} \\
\midrule

\multirow{7}{*}{InternVL2-8B}
& \cellcolor{lightgray} FP16            & \cellcolor{lightgray} $\times$ & \cellcolor{lightgray} N/A & \cellcolor{lightgray} 795 & \cellcolor{lightgray} 2211 & \cellcolor{lightgray} 77.71 & \cellcolor{lightgray} 91.05 & \cellcolor{lightgray} 97.07 \\
& SpikeMLLM(RTN)                        & $\checkmark$ & 255 & 635  & 1859 & 66.75 & 77.71 & 88.50 \\
& SpikeMLLM(GPTQ)                       & $\checkmark$ & 255 & 665  & 1870 & 71.36 & 85.61 & 90.83 \\
& SpikeMLLM(QuaRot)                     & $\checkmark$ & 7   & 432  & 1035 & 42.03 & 51.82 & 44.42 \\
& \cellcolor{lightcyan} SpikeMLLM(QuaRot+MSTS+TC-LIF) & \cellcolor{lightcyan} $\checkmark$ & \cellcolor{lightcyan} 2/3 & \cellcolor{lightcyan} 699 & \cellcolor{lightcyan} 2058 & \cellcolor{lightcyan} 72.41 & \cellcolor{lightcyan} 80.39 & \cellcolor{lightcyan} 92.71 \\
& SpikeMLLM(QuaRot)                     & $\checkmark$ & 15   & 772  & 2157 & 75.84 & 88.92 & 95.98 \\
& \cellcolor{lightcyan} SpikeMLLM(QuaRot+MSTS+TC-LIF) & \cellcolor{lightcyan} $\checkmark$ & \cellcolor{lightcyan} 3/4 & \cellcolor{lightcyan} \textbf{791} & \cellcolor{lightcyan} \textbf{2198} & \cellcolor{lightcyan} \textbf{76.98} & \cellcolor{lightcyan} \textbf{89.85} & \cellcolor{lightcyan} \textbf{96.83} \\
\midrule

\multirow{7}{*}{MiniCPM-V-2.6-8B}
& \cellcolor{lightgray} FP16            & \cellcolor{lightgray} $\times$ & \cellcolor{lightgray} N/A & \cellcolor{lightgray} 850 & \cellcolor{lightgray} 2223 & \cellcolor{lightgray} 78.26 & \cellcolor{lightgray} 89.93 & \cellcolor{lightgray} 96.58 \\
& SpikeMLLM(RTN)                        & $\checkmark$ & 255 & 810  & 1989 & 75.43 & 84.60 & 92.56 \\
& SpikeMLLM(GPTQ)                       & $\checkmark$ & 255 & 827  & 2076 & 77.21 & 87.62 & 95.14 \\
& SpikeMLLM(QuaRot)                     & $\checkmark$ & 7   & 35   & 668  & 13.84 & 6.89  & 16.38 \\
& \cellcolor{lightcyan} SpikeMLLM(QuaRot+MSTS+TC-LIF) & \cellcolor{lightcyan} $\checkmark$ & \cellcolor{lightcyan} 2/3 & \cellcolor{lightcyan} 739 & \cellcolor{lightcyan} 2002 & \cellcolor{lightcyan} 72.56 & \cellcolor{lightcyan} 76.35 & \cellcolor{lightcyan} 89.64 \\
& SpikeMLLM(QuaRot)                     & $\checkmark$ & 15   & 828  & 2069 & 77.35 & 87.04 & 94.15 \\
& \cellcolor{lightcyan} SpikeMLLM(QuaRot+MSTS+TC-LIF) & \cellcolor{lightcyan} $\checkmark$ & \cellcolor{lightcyan} 3/4 & \cellcolor{lightcyan} \textbf{829} & \cellcolor{lightcyan} \textbf{2197} & \cellcolor{lightcyan} \textbf{77.64} & \cellcolor{lightcyan} \textbf{88.80} & \cellcolor{lightcyan} \textbf{95.88} \\
\midrule

\multirow{7}{*}{Qwen-VL-Chat-9.6B}
& \cellcolor{lightgray} FP16            & \cellcolor{lightgray} $\times$ & \cellcolor{lightgray} N/A & \cellcolor{lightgray} 488 & \cellcolor{lightgray} 1852 & \cellcolor{lightgray} 60.47 & \cellcolor{lightgray} 57.96 & \cellcolor{lightgray} 67.58 \\
& SpikeMLLM(RTN)                        & $\checkmark$ & 255 & 432  & 1210 & 49.26 & 47.84 & 61.13 \\
& SpikeMLLM(GPTQ)                       & $\checkmark$ & 255 & 429  & 1648 & 54.41 & 50.12 & 63.06 \\
& SpikeMLLM(QuaRot)                     & $\checkmark$ & 7   & 10   & 42   & 0.10  & 0.26  & 9.18  \\
& \cellcolor{lightcyan} SpikeMLLM(QuaRot+MSTS+TC-LIF) & \cellcolor{lightcyan} $\checkmark$ & \cellcolor{lightcyan} 2/3 & \cellcolor{lightcyan} 392 & \cellcolor{lightcyan} 1617 & \cellcolor{lightcyan} 49.30 & \cellcolor{lightcyan} 43.01 & \cellcolor{lightcyan} 62.47 \\
& SpikeMLLM(QuaRot)                     & $\checkmark$ & 15   & 453  & 1731 & 54.78 & 51.51 & 63.06 \\
& \cellcolor{lightcyan} SpikeMLLM(QuaRot+MSTS+TC-LIF) & \cellcolor{lightcyan} $\checkmark$ & \cellcolor{lightcyan} 3/4 & \cellcolor{lightcyan} \textbf{479} & \cellcolor{lightcyan} \textbf{1815} & \cellcolor{lightcyan} \textbf{58.73} & \cellcolor{lightcyan} \textbf{55.42} & \cellcolor{lightcyan} \textbf{68.07} \\
\bottomrule
\end{tabular}
\end{adjustbox}
\end{table*}

\subsection{Co-designed Accelerator}
Importantly, TC-LIF does more than reduce the number of timesteps. By expressing quantized activations through temporally weighted binary spike components, it also restructures the underlying linear-layer computation into a multi-spike temporal accumulation process with shift-based temporal weighting. This means that, instead of relying on a general dense multiplication datapath, the induced execution can be realized through gated accumulation and lightweight sign handling. When combined with MSTS, which further lowers the effective temporal workload in multimodal inference, the resulting computation pattern becomes particularly attractive for deployment-oriented hardware specialization. This observation motivates the co-designed accelerator described next.
\begin{figure}[t]
  \centering
  \includegraphics[width=\linewidth]{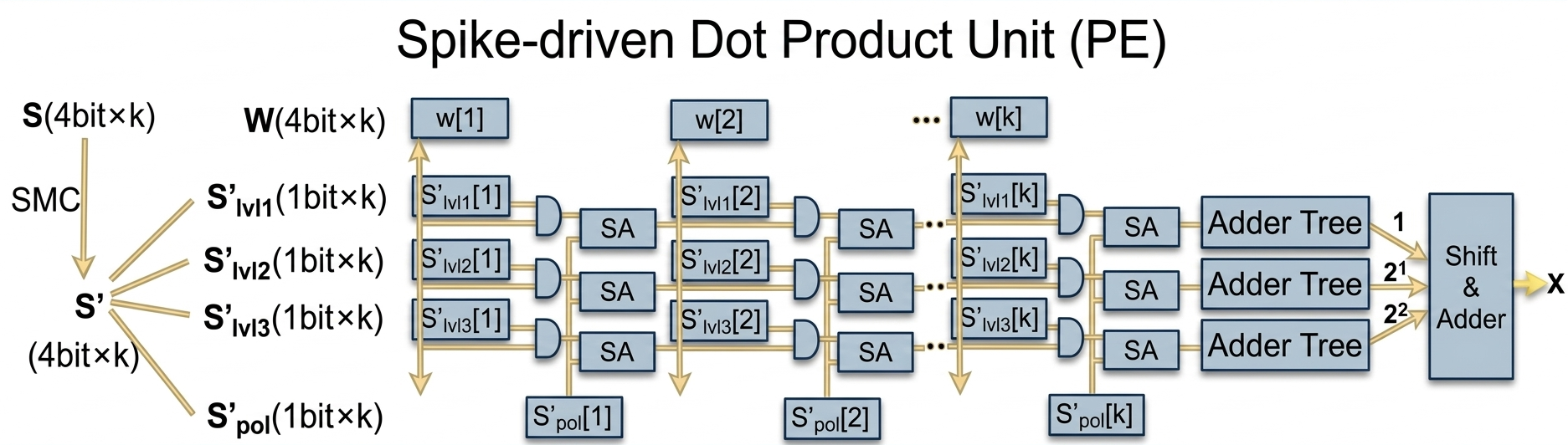}
  \caption{Microarchitecture of Spike-driven Dot-product PE}
  \label{fig:pemicro0}
\end{figure}

Fig.~\ref{fig:pemicro0} illustrates the microarchitecture of the Spike-driven Dot Product Unit (PE), the core compute primitive of the proposed co-designed accelerator, with $T=3$ as an example. The PE designs for other values of $T$ follow the same principle. The input spike vector $S$ is first converted from two's-complement form to its sign-magnitude representation $S'$ through sign-magnitude conversion (SMC), and then organized into four binary spike components. Among them, $S'_{\text{lvl1}}$, $S'_{\text{lvl2}}$, and $S'_{\text{lvl3}}$ represent the three binary significance levels of the input magnitude, while $S'_{\text{pol}}$ represents the spike polarity. In parallel, the weight vector $W=\{w[1], w[2], \dots, w[k]\}$ is processed element-wise along with the corresponding spike component. For each magnitude component, the contribution of $w[i]$ is conditionally activated by the associated spike bit through an AND operation, while sign adjustment (SA) logic, controlled by the polarity component, converts each selected term into either $+w[i]$ or $-w[i]$. The resulting signed terms are reduced by an adder tree to produce the partial sum of the current component. Partial sums from all magnitude components are then left-shifted according to the level index and accumulated through the Shift \& Adder module to generate the final dot-product result $X$. In this way, the PE realizes spike-driven dot products through a multi-spike temporal accumulation process, avoiding dedicated multipliers in favor of spike gating and significance-aware accumulation. 

Owing to its regular structure, this vector-level primitive can be naturally tiled into an array for matrix computation. Building on this computation pattern, we further construct a complete accelerator architecture for SpikeMLLM, while deferring the full system description to Supplementary Section~\ref{hardwaredetail}.

\section{Experiments}
For the exploration of introducing SNNs into multimodal large language models, we evaluate SpikeMLLM from three dimensions: (i) the compatibility of the framework with existing ANN quantization methods and overall performance on multimodal tasks; (ii) the accuracy–efficiency trade-off of MSTS modality-specific timestep allocation; (iii) the effectiveness of TC-LIF in preserving representational capacity under aggressive timestep compression.

\paragraph{\textbf{Models and Datasets}}
We evaluate SpikeMLLM on four multimodal large language models, including Qwen2VL-7B, InternVL2-8B, MiniCPM-V-2.6-8B~\cite{minicpmv}, and Qwen-VL-Chat-9.6B. Evaluations are conducted on five benchmarks: TextVQA~\cite{singh2019towards} and DocVQA~\cite{mathew2021docvqa} assess text understanding in natural scenes and document images, respectively; OCRBench~\cite{liu2023ocrbench} provides a comprehensive evaluation of OCR-related capabilities; MME~\cite{fu2023mme} evaluates models across 14 subtasks spanning perception and cognition; ScienceQA~\cite{lu2022learn} measures multimodal scientific reasoning ability. We further conduct scalability validation on Qwen2VL-72B.

\paragraph{\textbf{Baselines and Implementation Details}}
SpikeMLLM is compatible with existing ANN quantization methods, including RTN, GPTQ~\cite{frantar2023gptq}, and QuaRot~\cite{ashkboos2024quarot}. 
We adopt W4 quantization for LLM weights and W6 for the visual encoder.
We incorporate the conventional integer-to-spike unfolding method~\cite{yao2024spikev2, lei2025spike2former, yao2025scaling, xu2025neurormorphic} as the spiking baseline. As QuaRot is more robust to activation outliers and introduces minimal overhead, we further validate MSTS and TC-LIF based on SpikeMLLM(QuaRot). Following Remark~1, SpikeMLLM(QuaRot) at $T=15$ ($L=16$) serves as the spike-form reference corresponding to A4 activation quantization in our framework. 
The calibration dataset consists of 256 randomly selected samples from the training splits of the corresponding benchmarks~\cite{liu2023ocrbench, mathew2021docvqa, singh2019towards}.

\begin{table}[t]
\centering
\Huge
\setlength{\tabcolsep}{3pt}
\renewcommand{\arraystretch}{1.15}
\caption{Performance comparison of SpikeMLLM under different configurations on Qwen2VL-72B.}
\label{tab:72b}
\begin{adjustbox}{width=1\linewidth} 
\begin{tabular}{lccccccc}
\toprule
Method & Spiking & \makecell{{Time}\\{Step}} & OCRBench & MME & TextVQA & DocVQA & ScienceQA \\
\midrule
\rowcolor{lightgray}
FP16 & $\times$ & N/A & 831 & 2480 & 85.47 & 95.94 & 91.67 \\
SpikeMLLM(SQ) & $\checkmark$ & 255 & 678 & 2419 & 78.75 & 83.50 & 88.30 \\
SpikeMLLM(MBQ) & $\checkmark$ & 255 & 791 & 2331 & 83.09 & 94.41 & 87.80 \\
\rowcolor{lightcyan}
\makecell[l]{SpikeMLLM\\(QuaRot+MSTS+TC-LIF)} & $\checkmark$ & 3/4 & \textbf{817} & \textbf{2451} & \textbf{84.88} & \textbf{95.54} & \textbf{89.84} \\
\bottomrule
\end{tabular}
\end{adjustbox}
\end{table}

\begin{table}[t]
\centering
\Huge
\setlength{\tabcolsep}{3pt}
\renewcommand{\arraystretch}{1.1}
\caption{Analysis of MED-guided layer-wise timestep allocation on Qwen2VL-7B and InternVL2-8B. MSTS-reverse assigns higher timesteps to lower-MED layers under the same budget. Additional results are provided in Supplementary Table~\ref{tab:reverse}.}
\label{tab:msts_layer}
\resizebox{\linewidth}{!}{
\begin{tabular}{lcccccccc}
\toprule
Method & Spiking & \makecell{{Time}\\{Step}} & OCRBench & MME & TextVQA & DocVQA & ScienceQA & Avg gap \\
\midrule
\rowcolor{lightgray}
Qwen2VL-7B-FP16 & $\times$ & N/A & 835 & 2285 & 84.11 & 93.95 & 85.67 & -- \\
MSTS & $\checkmark$ & 2/3 & 717 & 1984 & 78.28 & 84.20 & 80.37 & 10.15 \\
MSTS & $\checkmark$ & 2.3/3.3 & 765 & 2109 & 80.97 & 89.12 & 81.95 & 5.86 \\
MSTS-reverse & $\checkmark$ & 2.3/3.3 & 752 & 2083 & 80.07 & 87.56 & 81.71 & 7.00 \\
MSTS & $\checkmark$ & 3/4 & 823 & 2236 & 83.46 & 93.31 & 85.57 & 1.03 \\
\midrule
\rowcolor{lightgray}
InternVL2-8B-FP16 & $\times$ & N/A & 795 & 2211 & 77.71 & 91.05 & 97.07 & -- \\
MSTS & $\checkmark$ & 2/3 & 699 & 2058 & 72.41 & 80.39 & 92.71 & 8.40 \\
MSTS & $\checkmark$ & 2.3/3.3 & 741 & 2117 & 74.39 & 85.66 & 94.99 & 4.67 \\
MSTS-reverse & $\checkmark$ & 2.3/3.3 & 723 & 1980 & 73.63 & 81.49 & 93.36 & 7.82 \\
MSTS & $\checkmark$ & 3/4 & 791 & 2198 & 76.98 & 89.85 & 96.83 & 0.72 \\
\bottomrule
\end{tabular}}
\end{table}

\subsection{Main Results}

\paragraph{\textbf{Overall Performance}}
SpikeMLLM achieves near-lossless performance under aggressive timestep reduction across multiple MLLM\allowbreak s. Table~\ref{tab:main_results} presents the overall performance across four MLLMs. SpikeMLLM(RTN) and SpikeMLLM(GPTQ) achieve competitive accuracy at $T=255$, validating the compatibility of the framework with diverse ANN quantization methods. However, their reliance on extremely long timesteps leads to inefficient temporal unfolding, diminishing the efficiency advantage of spike-driven computation. SpikeMLLM(QuaRot) achieves competitive performance at $T=15$; however, at the low timestep configuration $T=7$, all models suffer significant performance degradation --- for instance, the OCRBench score of MiniCPM-V-2.6-8B drops to 35 --- indicating that aggressively reducing timesteps under conventional integer-to-spike unfolding significantly degrades representation capacity in multimodal settings. After incorporating MSTS and TC-LIF, at $T_v/T_t=3/4$, the average performance gap to the FP16 baseline (across five benchmarks) is only 1.03\% on Qwen2VL-7B and 0.72\% on InternVL2-8B, while reducing the timestep from 15 to $3/4$. This trend holds consistently across all four models. At $T_v/T_t=2/3$, the performance degradation remains controlled despite further timestep compression, demonstrating a favorable accuracy--efficiency trade-off.

\paragraph{\textbf{Performance on Large Models}}
Table~\ref{tab:72b} presents the results of SpikeMLLM on Qwen2VL-72B. SpikeMLLM maintains strong scalability to large models, achieving only a 1.19\% average gap to the FP16 baseline at $T_v/T_t=3/4$. We further validate the compatibility of SpikeMLLM with quantization methods including SmoothQuant (SQ)~\cite{xiao2023smoothquant} and MBQ~\cite{li2025mbq} at the 72B scale, where MBQ is an ANN quantization method specifically designed for multimodal models. SpikeMLLM (QuaRot+MSTS+TC-LIF) consistently outperforms SpikeMLLM(MBQ) across all five benchmarks, demonstrating that SpikeMLLM maintains strong compatibility and effectiveness at larger model scales.

\paragraph{\textbf{Fine-grained Timestep Allocation}}
Table~\ref{tab:msts_layer} reports performance under different $T_v/T_t$ budgets, where fractional values (e.g., 2.x/3.x) indicate layers mix two adjacent timestep levels. Building on modality-level allocation, fine-grained layer-wise mixing further reduces the average relative gap to FP16 under similar effective timesteps: for Qwen2VL-7B, $T_v/T_t=2/3 \rightarrow 2.3/3.3$ narrows the gap from 10.15\% to 5.86\%, with consistent trends on InternVL2-8B. Under the same budget, MSTS-reverse consistently underperforms MED-guided allocation, confirming that the gain comes from MED-guided layer ordering rather than budget level alone.

\subsection{Ablation Study}
We evaluate MSTS and TC-LIF against the specific question each component addresses. 
MSTS is ablated across a continuous range of modality-specific allocation ratios 
at two budget levels (Figs.~\ref{fig:msts_high}--\ref{fig:msts_low}).
TC-LIF is ablated under a fixed timestep budget (Table~\ref{tab:tclif_ablation}) 
and across increasing timestep allocations (Fig.~\ref{fig:tclif}); we further analyze spike-driven efficiency in terms of firing rate and FLOPs reduction.

\paragraph{\textbf{Effect of MSTS}}
We analyze the effect of different modality-specific timestep allocation ratios $T_v/T_t$ on Qwen2VL-7B and InternVL2\allowbreak -8B, with results shown in Figs.~\ref{fig:msts_high} and~\ref{fig:msts_low}. Experiments demonstrate that MSTS achieves a superior performance--efficiency trade-off by redistributing timesteps across modalities. During prefill, where visual tokens substantially outnumber text tokens, the effective inference cost is largely determined by $T_v$; thus $T_v/T_t = 3/4$ incurs similar inference cost to equal allocation ($T_v = T_t = 3$), while $T_v/T_t = 4/3$ is comparable in cost to equal allocation ($T_v = T_t = 4$). At equal or lower inference cost, $T_v/T_t = 3/4$ consistently outperforms equal allocation ($T_v = T_t = 3$) across all five benchmarks on both models, and also surpasses the more expensive $T_v/T_t = 4/3$; specifically, its average performance gap relative to FP16 is only 1.03\% and 0.72\% on Qwen2VL-7B and InternVL2-8B respectively, compared to 3.28\% and 2.34\% for $T_v/T_t = 4/3$, demonstrating that better performance can be achieved at lower computational cost. Under low timestep budgets ($T_v \in \{2, 3\}$), this effect is even more pronounced: insufficient text timesteps (e.g., $T_v/T_t = 3/2$) cause severe performance degradation (OCRBench: 32, DocVQA: 5.31), whereas allocating more timesteps to the text modality at the same cost (e.g., $T_v/T_t = 2/3$) substantially recovers performance (OCRBench: 717, DocVQA: 84.20). These results consistently demonstrate that the text modality is more sensitive to timestep allocation, and transferring timesteps from visual to text while keeping the prefill cost constant is the key to achieving better performance, a finding that directly motivates the design of MSTS.

\begin{figure}[t]
  \centering
  \includegraphics[width=\linewidth]{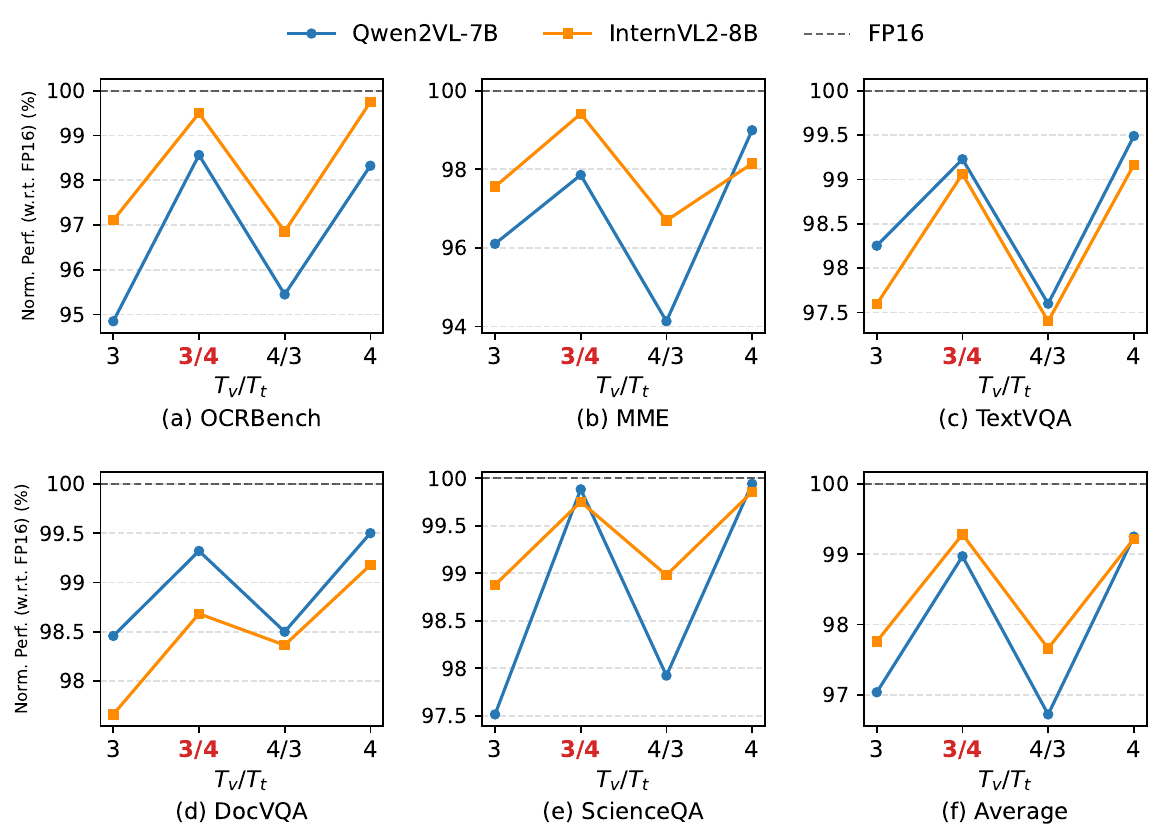}
  \caption{Ablation study of MSTS: Normalized performance of SpikeMLLM(QuaRot+MSTS+TC-LIF) under different modality-specific timestep allocations $T_v/T_t$ ($T \in \{3, 4\}$) on Qwen2VL-7B and InternVL2-8B.
  }
  \label{fig:msts_high}
\end{figure}

\begin{figure}[t]
  \centering
  \includegraphics[width=\linewidth]{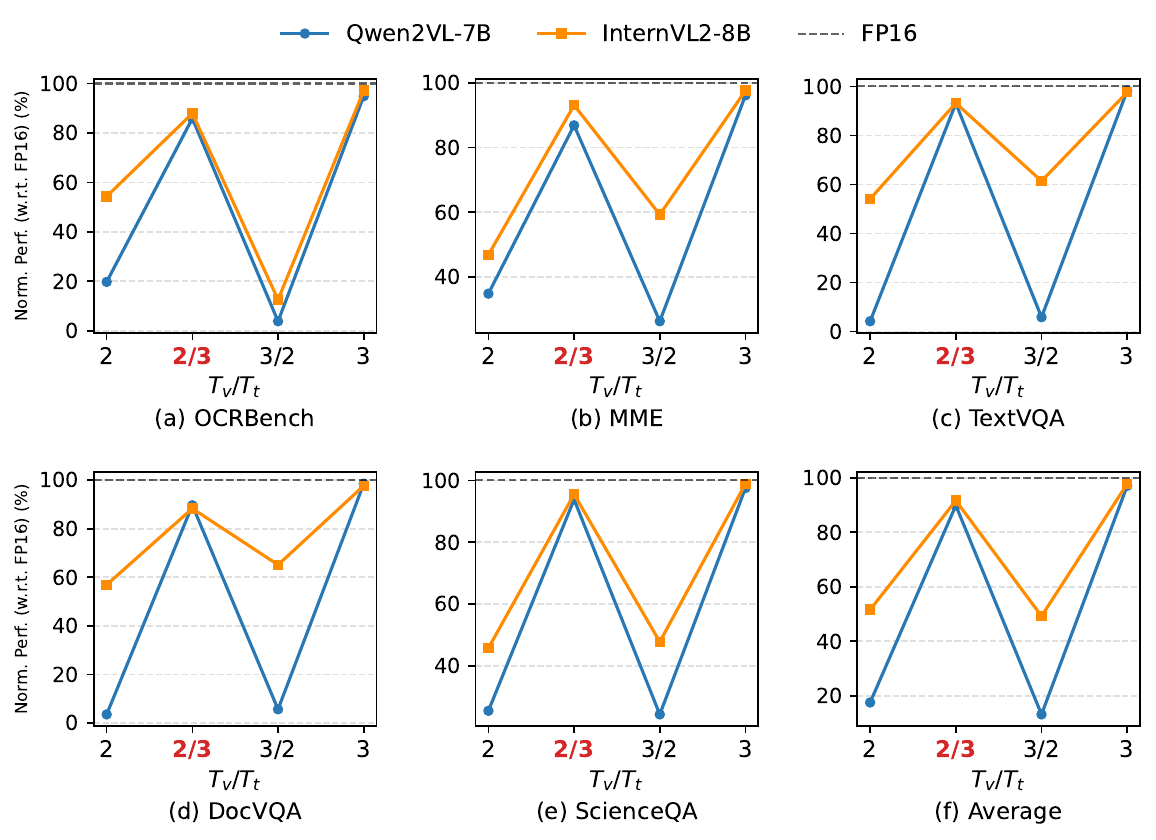}
  \caption{Ablation study of MSTS: Normalized performance of SpikeMLLM(QuaRot+MSTS+TC-LIF) under different modality-specific timestep allocations $T_v/T_t$ ($T \in \{2, 3\}$) on Qwen2VL-7B and InternVL2-8B.}
  \label{fig:msts_low}
\end{figure}

\paragraph{\textbf{Effect of TC-LIF}}
We first validate the necessity of TC-LIF by comparing it against standard integer-to-spike unfolding at the same timestep budget $T_v/T_t = 3/4$, with results shown in Table~\ref{tab:tclif_ablation}. Under standard unfolding, performance collapses completely on both models (OCRBench: 2, TextVQA: 0 on Qwen2VL-7B; OCRBench: 0, TextVQA: 0 on InternVL2-8B), since standard unfolding at T=3 is insufficient to preserve representational capacity. TC-LIF recovers performance to near-FP16 levels at the same timestep budget, confirming that temporal weighted integration is the key to enabling aggressive timestep compression without sacrificing representational capacity.

Fig.~\ref{fig:tclif} shows the normalized performance of Qwen2VL-7B and InternVL2-8B across benchmarks under different $T_v/T_t$ configurations relative to FP16. In the range from $T_v/T_t = 2/3$ to $T_v/T_t = 3/4$, performance improves notably across all benchmarks as the timestep budget increases, gradually approaching the FP16 baseline. $T_v/T_t = 3/4$ is close to the performance saturation point, beyond which the marginal performance gain from additional timesteps diminishes --- from $T_v/T_t = 3/4$ to $T_v/T_t = 7$, the average relative gap across five benchmarks narrows only from 1.03\% to 0.32\% on Qwen2VL-7B, while the effective timestep increases considerably. Furthermore, the convergence trends across benchmarks are highly consistent on both models, demonstrating the stability of the TC-LIF temporal compression mechanism across different types of multimodal tasks. These results suggest that $T_v/T_t = 3/4$ achieves a favorable trade-off between performance and efficiency, and we adopt it as the default configuration.

\begin{table}[t]
\centering
\Large
\setlength{\tabcolsep}{3pt}
\renewcommand{\arraystretch}{1.1}
\caption{Ablation of TC-LIF against standard integer-to-spike unfolding at $T_v/T_t = 3/4$.}
\label{tab:tclif_ablation}
\begin{adjustbox}{width=1\linewidth} 
\begin{tabular}{lc cc cc}
\toprule
\multirow{2}{*}{Method} & \multirow{2}{*}{Time Step} & \multicolumn{2}{c}{Qwen2VL-7B} & \multicolumn{2}{c}{InternVL2-8B} \\
\cmidrule(lr){3-4} \cmidrule(lr){5-6}
& & OCRBench & TextVQA & OCRBench & TextVQA \\
\midrule
\rowcolor{lightgray}
FP16 & N/A & 835 & 84.11 & 795 & 77.71  \\
SpikeMLLM(QuaRot+MSTS) & 3/4 & 2 & 0 & 0 & 0 \\
SpikeMLLM(QuaRot+MSTS+TC-LIF) & 3/4 & 823 & 83.46 & 791 & 76.98 \\
\bottomrule
\end{tabular}
\end{adjustbox}
\end{table}

\begin{figure}[t]
  \centering
  \includegraphics[width=\linewidth]{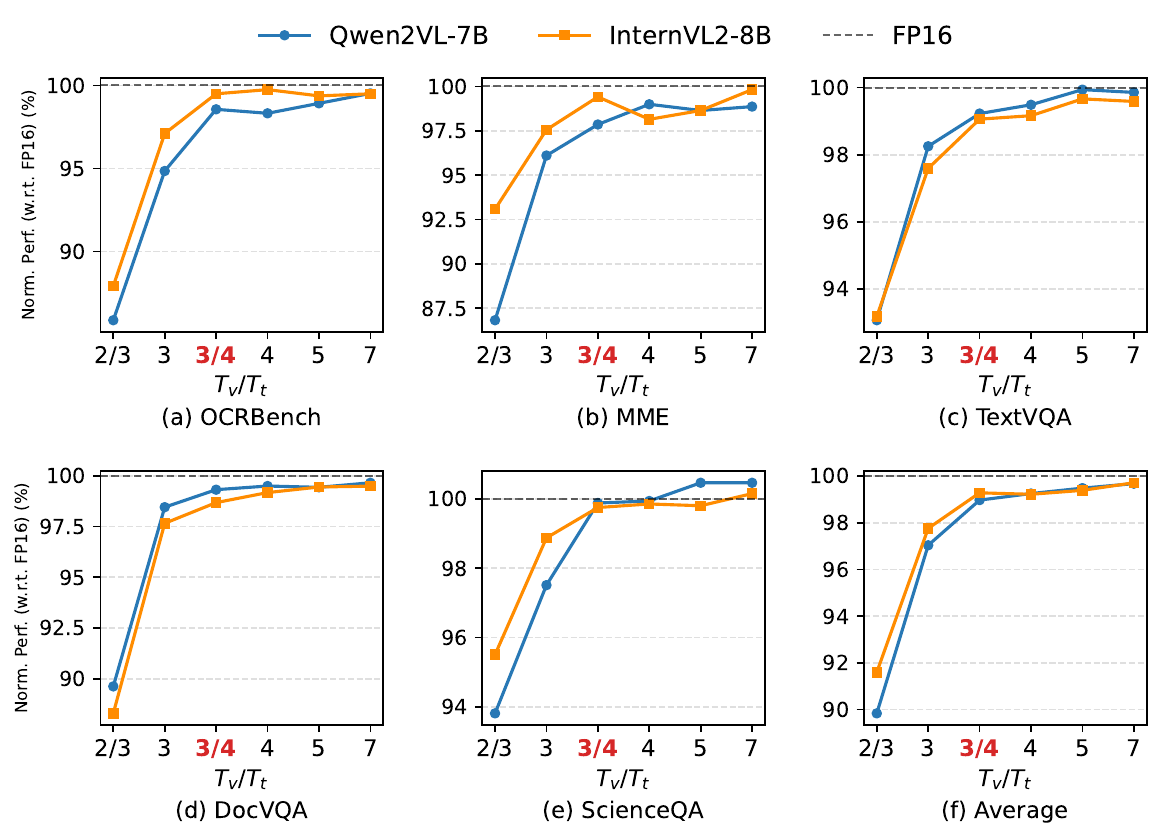}
  \caption{Ablation study of TC-LIF: Normalized performance of SpikeMLLM(QuaRot +MSTS+TC-LIF) as timestep $T_v/T_t$ increases on Qwen2VL-7B and InternVL2-8B.}
  \label{fig:tclif}
\end{figure}

\paragraph{\textbf{Spiking  Efficiency Analysis}}
Fig.~\ref{fig:firing_rate} shows the per-layer firing rate distribution on Qwen2VL-7B under two configurations. SpikeMLLM(QuaRot) ($T=15$) exhibits an average firing rate of 0.53, while SpikeMLLM(QuaRot+MSTS+TC-LIF) ($T_v/T_t=3/4$) reduces it to 0.31, reflecting the reduced temporal workload under MSTS and TC-LIF. Table~\ref{tab:efficiency} reports FLOPs on Qwen2VL-7B under a representative setting of 1024$\times$1024 resolution input and 50 text tokens, where SpikeMLLM(QuaRot+MSTS+TC-LIF) ($T_v/T_t=3/4$) achieves a 4.4$\times$ reduction compared to SpikeMLLM(QuaRot) ($T=15$); consistent results on InternVL2-8B are provided in Supplementary Table~\ref{tab:efficiency_intern}. FLOPs for FP16 are computed as multiply-accumulate (MAC) operations; for spike-based methods, weight-activation interactions reduce to accumulate-only (AC) operations owing to binary spikes. Full per-layer firing rate distributions for all methods on Qwen2VL-7B and InternVL2-8B are provided in Supplementary Figs.~\ref{fig:firing_rate_full_qwen} and~\ref{fig:firing_rate_full_intern}.

\begin{figure}[t]
  \centering
  \begin{subfigure}[b]{\linewidth}
    \centering
    \includegraphics[width=0.9\linewidth]{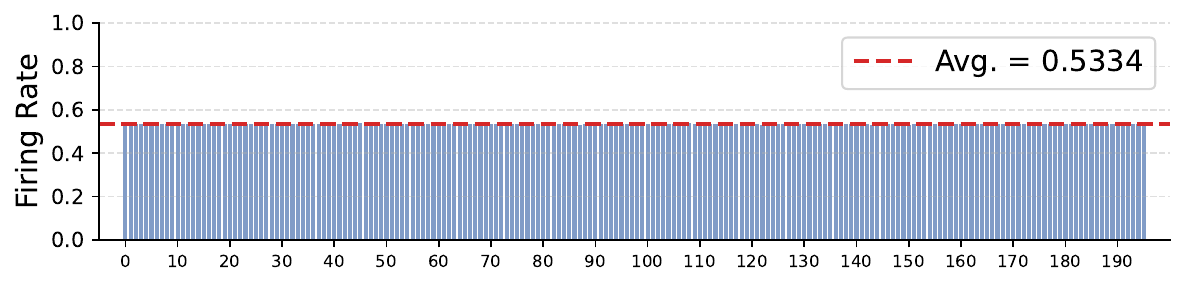}
    \caption{SpikeMLLM(QuaRot) ($T=15$)} 
    \label{fig:firing_rate_a}
  \end{subfigure}
  
  \vspace{0.5em} 

  \begin{subfigure}[b]{\linewidth}
    \centering
    \includegraphics[width=0.9\linewidth]{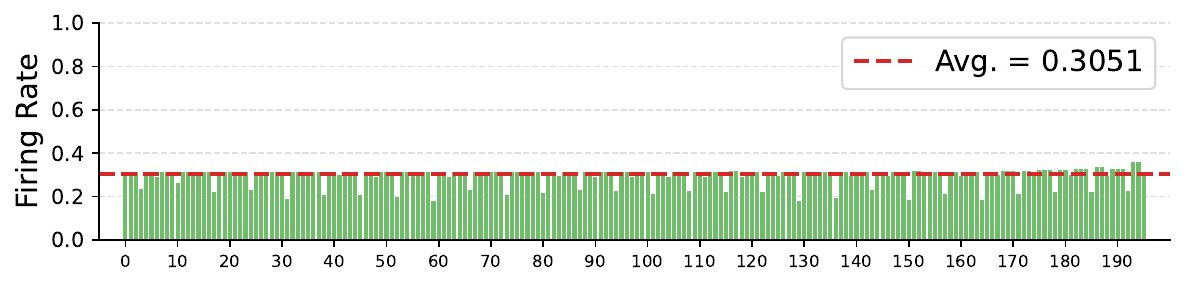}
    \caption{SpikeMLLM(QuaRot+MSTS+TC-LIF) ($T_v/T_t=3/4$)}
    \label{fig:firing_rate_b}
  \end{subfigure}

  \caption{Firing rate distribution across layers for SpikeMLLM(QuaRot) ($T=15$) and SpikeMLLM(QuaRot+MSTS+TC-LIF) ($T_v/T_t=3/4$) on Qwen2VL-7B.} 
  \label{fig:firing_rate}
\end{figure}

\begin{table}[t]
\centering
\small
\renewcommand{\arraystretch}{0.8}
\caption{Firing rate and FLOPs analysis of SpikeMLLM on Qwen2VL-7B. FLOPs are computed as INT FLOPs following \citet{xu2023qdetr,liu2020birealnet,xu2025neurormorphic}; see Supplementary Section~\ref{appendix:flops} for computation details.}
\label{tab:efficiency}
\begin{adjustbox}{width=1\linewidth} 
\begin{tabular}{lccc}
\toprule
Method & Time Step & Firing Rate & FLOPs (T) \\
\midrule
FP16 & N/A & N/A & 8.78 \\
SpikeMLLM(RTN) & 255 & 0.50 & 17.57 \\
SpikeMLLM(GPTQ) & 255 & 0.50 & 17.57 \\
SpikeMLLM(QuaRot) & 7 & 0.57 & 0.55 \\
SpikeMLLM(QuaRot+MSTS+TC-LIF) & 2/3 & 0.30 & 0.17 \\
W4A4 & N/A & N/A & 0.55 \\
SpikeMLLM(QuaRot) & 15 & 0.53 & 1.10 \\
SpikeMLLM(QuaRot+MSTS+TC-LIF) & 3/4 & 0.31 & 0.25 \\
\bottomrule
\end{tabular}
\end{adjustbox}
\end{table}

\subsection{Deployment-oriented Co-designed System Study}

The proposed accelerator is used to evaluate the deployment of the low-timestep SpikeMLLM configuration based on QuaRot+MSTS+TC\allowbreak-LIF. At the PE level, the current RTL prototype implements the dominant 3-bit temporal-magnitude datapath induced by TC-LIF, while modality-aware control at the system level supports the asymmetric timestep allocation policy of SpikeMLLM. We implement the RTL design for SpikeMLLM and use Synopsys DC on SMIC 28nm CMOS technology to estimate logic area and power. The on-chip SRAM buffer macros are generated using the ARM ARTISAN 28nm single-port SRAM compiler. Off-chip HBM behavior is modeled by simulating row activation and access patterns under different data layouts. Memory latency and IO power are estimated following ~\cite{MCBP}. We extract per-stage cycle counts from RTL simulation using Synopsys VCS, and use a custom cycle-level simulator to evaluate overall execution latency. Detailed post-synthesis hardware implementation metrics of the proposed accelerator are listed in Supplementary Section~\ref{RTL}. 

On the GPU side, we deploy the Qwen2VL-7B FP16 baseline on an NVIDIA A800 under the same batch-size-1 setting, and use it as the deployment reference. We report the measured inference throughput and board power under the evaluated configuration. 
As shown in Table~\ref{accel}, the proposed design shows 9.06$\times$ higher throughput (393.7 token/s vs.\ 43.5 token/s). In terms of system power, the proposed design operates at 7.13 W, compared with 184 W for the GPU reference, corresponding to 25.8$\times$ lower power. These gains arise from the combination of algorithmic workload reduction and hardware specialization.

These results indicate that SpikeMLLM does more than preserve multimodal accuracy under compressed timesteps. MSTS and TC-LIF reformulate the computation into a spike-driven pattern with reduced temporal depth, lower activity, and more regular low-bit execution. Such properties are favorable for accelerator implementation, enabling the simplifications introduced at the model level to be reflected in system-level efficiency. Hardware evaluation results support this observation and suggest that software–hardware co-design is an effective approach for improving deployment efficiency.

\begin{table}[t]
\centering
\caption{Deployment-oriented throughput and power comparison between the proposed SpikeMLLM accelerator for the low-timestep QuaRot+MSTS+TC-LIF setting and Qwen2VL-7B FP16 inference on NVIDIA A800.}
\Huge
\label{accel}
\begin{adjustbox}{width=1\linewidth} 
\begin{tabular}{lccc}
\toprule
System & Power [W] & Throughput [token/s] \\
\midrule
SpikeMLLM (Proposed accelerator) & 7.13 & 393.7 \\
Qwen2VL-7B-FP16 (NVIDIA A800) & 184 & 43.5 \\
Relative gain &  25.8× &  9.06×  \\
\bottomrule
\end{tabular}
\end{adjustbox}
\end{table}

\section{Conclusion}
In this paper, we propose SpikeMLLM, a spike-based framework for  MLLMs that consistently models existing ANN quantization methods within the spiking representation space. To address the modality-heterogeneous temporal scale requirements and high timestep unfolding overhead encountered when extending SNNs to  MLLMs, and in contrast to the uniform timestep configurations and linearly growing integer-to-spike unfolding mechanisms in existing methods, we further propose MSTS and TC-LIF, achieving modality-aware differentiated timestep allocation and timestep compression from $T=L-1$ to $T=\log_2(L)-1$, respectively. Experimental results demonstrate that SpikeMLLM maintains near-FP16 performance while significantly compressing timesteps across multiple representative MLLMs and multimodal benchmarks, with scalability further validated on Qwen2VL-72B. 
Beyond preserving multimodal accuracy under aggressive timestep compression, the proposed formulation reshapes inference into a low-timestep, multiplier-free-friendly computation pattern. Our deployment-oriented RTL study suggests that such algorithmic simplifications can be translated into $9.06\times$ throughput and $25.8\times$ power benefits when paired with dedicated hardware.

\bibliographystyle{ACM-Reference-Format}
\bibliography{main}

\clearpage
\newcounter{savepage}
\setcounter{savepage}{\value{page}}

\title{Supplementary Materials:\\ SpikeMLLM: Spike-based Multimodal Large Language Models via Modality-Specific Temporal Scales and Temporal Compression}

\maketitlesupp

\setcounter{page}{\value{savepage}}

\setcounter{section}{5}
\renewcommand{\thesection}{\arabic{section}}
\setcounter{figure}{7}  
\setcounter{table}{6}    

\section{Related Work}
\subsection{\textbf{From Spiking Neural Networks to Spike-based Language and Multimodal Models}}
In recent years, SNNs~\cite{maass1997networks,fang2023spikingjelly,roy2019towards} have expanded from early visual recognition~\cite{cao2015spiking,yao2025scaling,lei2025spike2former} to sequence modeling and language tasks, with researchers exploring direct training~\cite{wu2018spatio,eshraghian2023training,fang2021deep,zheng2021going,neftci2019surrogate}, ANN-to-SNN conversion~\cite{hu2023spiking,rathi2020enabling,Hu2023FastSNNFS,wu2021progressive,wang2023masked}, and integer-to-spike unfolding \cite{yao2024spikev2, lei2025spike2former, yao2025scaling, xu2025neurormorphic,luo2024integer}. Early works attempted to emit integer spikes for richer representations~\cite{Hu2023FastSNNFS,xing2025spikellm}, but typically violated the spike-driven sparse computation property. The integer-to-spike unfolding paradigm instead maps membrane potentials to discrete spike counts, unfolded into multi-timestep binary sequences that preserve event-driven sparse addition while enhancing representational capacity~\cite{yao2024spikev2,xu2025neurormorphic}. Its structural consistency with quantized ANN activations makes it a key direction for scalable spike-based modeling. However, this paradigm requires $T = L-1$ timesteps for $L$ quantization levels, and the temporal overhead grows linearly with quantization precision, becoming more pronounced in multimodal scenarios with larger token counts.

In the direction of language modeling, SpikeGPT~\cite{zhu2024spikegpt}, SpikingBERT~\cite{bal2024spikingbert}, and SpikeLM~\cite{xing2024spikelm} initially validated the feasibility of spike-based language modeling, but remained confined to small-scale supervised tasks. SpikeLLM~\cite{xing2025spikellm} and NSLLM~\cite{xu2025neurormorphic} further scaled SNNs to billion-parameter LLMs, achieving competitive performance on general language tasks. In the direction of multimodal perception, SNNs have been explored for audio-visual classification~\cite{liu2022event, guo2023transformer, liu2025misnet}, emotion recognition~\cite{tan2020fusionsense, guo2024smile, chen2025enhancing}, event-frame fusion~\cite{zhou2024enhancing}, and hybrid ANN-SNN architectures~\cite{sun2024reliable, tumpa2024snn}, primarily targeting cross-modal feature fusion, temporal alignment, or modality imbalance in small-scale perceptual tasks.

However, these works share two common limitations: existing spike-based LLM works remain confined to unimodal text scenarios, while existing multimodal SNN works address explicit multi-branch fusion rather than the modality-heterogeneous representational efficiency problem within MLLMs. In summary, existing works either remain confined to unimodal text scenarios or stay within small-scale multimodal perceptual tasks, leaving the introduction of SNNs into MLLMs unexplored — this is the starting point of our work.

\subsection{\textbf{ANN Quantization and Its Extension to MLLMs}}
Quantization reduces the computational and memory overhead of large models by compressing weights and activations into low-bit representations, and has become an important approach for efficient inference. Depending on whether retraining is required, quantization methods are broadly categorized into Post-Training Quantization (PTQ) and Quantization-Aware Training (QAT)~\cite{liu2023llmqat,jacob2018quantization}. PTQ requires no backpropagation and operates directly on a small calibration dataset, making it particularly practical for large-scale models; representative methods such as GPTQ~\cite{frantar2023gptq}, SmoothQuant~\cite{xiao2023smoothquant}, and QuaRot~\cite{ashkboos2024quarot} have achieved significant progress in LLM quantization. However, these methods are designed for unimodal language models and do not account for the heterogeneity between vision and language tokens in multimodal scenarios. To address this, MBQ~\cite{li2025mbq} identifies significant differences in quantization sensitivity between vision and language tokens and proposes a modality-aware calibration strategy. Nevertheless, existing quantization methods all operate within the ANN representation space. The key insight of this work is to consistently model the results of ANN quantization within the spiking representation space, enabling compatibility with existing quantization methods while introducing the energy efficiency advantages of the spike-driven paradigm.

\subsection{\textbf{Neuromorphic Hardware and Spike-driven Acceleration}}
Neuromorphic hardware is inspired by the structure and function of the biological brain, featuring compute-memory co-location and spike-driven event-based computation~\cite{roy2019towards,schuman2022opportunities}, which offers significant energy efficiency advantages over conventional von Neumann architectures. Existing neuromorphic platforms span both pure SNN architectures~\cite{merolla2014million,davies2018loihi} and hybrid ANN-SNN designs~\cite{ma2022neuromorphic,hoppner2021spinnaker} to improve flexibility. Among recent representative systems, the asynchronous sensing-computing SoC Speck demonstrates ultra-low static power consumption of approximately 0.42~mW, with 0.7--15~mW in typical neuromorphic vision tasks~\cite{yao2024spike}. Intel Loihi~2 significantly enhances programmability and representational flexibility through microcode-defined neuron state transitions and firing logic, as well as graded spike events supporting different polarities or amplitudes~\cite{intel2021loihi2}. In the direction of large model deployment, recent work has mapped a MatMul-free LLM with ternary weights and 16-bit activations onto a multi-chip Loihi~2 system, achieving up to 10$\times$ energy reduction and 4$\times$ throughput improvement over edge GPUs via asynchronous processing~\cite{zhu2024matmulfree}. From a hardware--algorithm co-design perspective, NSLLM develops a MatMul-free FPGA accelerator for a 1.5B-scale spiking LLM under a large timestep setting, achieving 19.8$\times$ energy-efficiency improvement and 2.2$\times$ throughput gain over a GPU baseline~\cite{xu2025neurormorphic}. In this work, we design a dedicated RTL accelerator for SpikeMLLM to validate its energy efficiency, providing a reference for algorithm-hardware co-design of next-generation neuromorphic chips.

\section{Theoretical Proofs}
\subsection{Proof of the Reference Equivalence under Standard Integer-to-Spike Unfolding}
\label{app:reference_equivalence}

\begin{proposition}
\textbf{(Reference Equivalence between Standard Integer-to-Spike Unfolding and $A$-bit Activation Quantization).}
For an $A$-bit activation quantizer, let the number of quantization levels be $L=2^A$, and let the quantized activation at layer $l$ be denoted by
\begin{equation}
\mathbf{S}^l \in \{0,1,\ldots,L-1\}.
\end{equation}
Under standard integer-to-spike unfolding, we set
\begin{equation}
T=L-1,
\end{equation}
and expand each integer activation $\mathbf{S}^l=k$ into a binary spike sequence of length $T$:
\begin{equation}
\hat{\mathbf{S}}^l[t] =
\begin{cases}
1, & 1 \le t \le k,\\
0, & k < t \le T,
\end{cases}
\qquad t=1,\ldots,T.
\end{equation}
Then standard unfolding is equivalent to the original $A$-bit activation quantization in both numerical representation and linear-layer computation, i.e.,
\begin{equation}
\sum_{t=1}^{T}\hat{\mathbf{S}}^l[t] = \mathbf{S}^l,
\end{equation}
and
\begin{equation}
\mathbf{X}^{l+1}=\mathbf{W}^{l+1}\mathbf{S}^l
=
\mathbf{W}^{l+1}\sum_{t=1}^{T}\hat{\mathbf{S}}^l[t]
=
\sum_{t=1}^{T}\mathbf{W}^{l+1}\hat{\mathbf{S}}^l[t].
\end{equation}
Therefore, the $T=L-1$ configuration serves as the spike-form reference corresponding to $A$-bit activation quantization.
\end{proposition}

\begin{proof}
For an $A$-bit activation quantizer, the discrete support contains $L=2^A$ levels, so the quantized activation can be represented as an integer
\begin{equation}
\mathbf{S}^l \in \{0,1,\ldots,L-1\}.
\end{equation}
Under standard integer-to-spike unfolding, the number of timesteps is chosen as
\begin{equation}
T=L-1,
\end{equation}
which is exactly sufficient to represent the maximum integer value $L-1$ using binary spikes unfolded along time.

Consider an arbitrary quantized activation value $\mathbf{S}^l=k$, where $k\in\{0,1,\ldots,L-1\}$. By construction, its unfolded binary spike sequence is defined as
\begin{equation}
\hat{\mathbf{S}}^l[t] =
\begin{cases}
1, & 1 \le t \le k,\\
0, & k < t \le T,
\end{cases}
\qquad t=1,\ldots,T.
\end{equation}
Hence, among the $T$ timesteps, exactly the first $k$ entries are equal to $1$, and the remaining $T-k$ entries are equal to $0$. Therefore,
\begin{equation}
\sum_{t=1}^{T}\hat{\mathbf{S}}^l[t]
=
\underbrace{1+\cdots+1}_{k\ \text{times}}
+
\underbrace{0+\cdots+0}_{T-k\ \text{times}}
=
k
=
\mathbf{S}^l.
\end{equation}
This proves that standard unfolding preserves the numerical value of the original quantized activation exactly.

Next, consider the linear transform at the next layer. Let $\mathbf{W}^{l+1}$ denote the weight matrix of layer $l+1$. Since matrix multiplication is linear, we have
\begin{equation}
\mathbf{X}^{l+1}
=
\mathbf{W}^{l+1}\mathbf{S}^l
=
\mathbf{W}^{l+1}\sum_{t=1}^{T}\hat{\mathbf{S}}^l[t]
=
\sum_{t=1}^{T}\mathbf{W}^{l+1}\hat{\mathbf{S}}^l[t].
\end{equation}
Thus, the output produced by standard integer-to-spike unfolding is exactly the same as that produced by the original quantized activation under the same linear layer.

Combining the two results above, standard integer-to-spike unfolding with $T=L-1$ is equivalent to the original $A$-bit activation quantization in both numerical representation and linear-layer computation. Therefore, it serves as the spike-form reference corresponding to $A$-bit activation quantization.
\end{proof}

In particular, for A4 activation quantization, we have $L=2^4=16$, and thus the corresponding standard unfolding uses $T=L-1=15$. Therefore, the $T=15$ configuration used in our framework serves as the spike-form reference corresponding to standard A4 activation quantization. The main goal of MSTS and TC-LIF is not to reproduce this reference, but to preserve its performance as much as possible under aggressive timestep compression.

\subsection{Proof of the Representational Equivalence under TC-LIF}
\label{app:tclif_equivalence}

\begin{proposition}
\textbf{(Representational Equivalence of Polar TC-LIF to the Effective Symmetric $A$-bit Discrete Set).}
For an $A$-bit discrete activation space, let the number of quantization levels be $L=2^A$.  
Under the polar TC-LIF formulation, after dropping the unreachable extreme level $-L/2$, the effective symmetric discrete support becomes
\begin{equation}
\mathbf{S}^l \in \left\{-\frac{L}{2}+1,\,-\frac{L}{2}+2,\ldots,\frac{L}{2}-1\right\}.
\end{equation}
Let the number of timesteps be
\begin{equation}
T=\log_2(L)-1=A-1.
\end{equation}
Then every quantized activation $\mathbf{S}^l$ in the above set can be represented exactly by a polar weighted spike sequence
\begin{equation}
\mathbf{S}^l=\sum_{t=1}^{T}2^{t-1}\hat{\mathbf{S}}^l[t],
\qquad \hat{\mathbf{S}}^l[t]\in\{-1,0,1\},
\end{equation}
where the spike polarity carries the sign and the temporal weights encode the binary magnitude. Moreover, for a linear layer with weight matrix $\mathbf{W}^{l+1}$,
\begin{equation}
\mathbf{X}^{l+1}=\mathbf{W}^{l+1}\mathbf{S}^l
=
\mathbf{W}^{l+1}\sum_{t=1}^{T}2^{t-1}\hat{\mathbf{S}}^l[t]
=
\sum_{t=1}^{T}2^{t-1}\mathbf{W}^{l+1}\hat{\mathbf{S}}^l[t].
\end{equation}
Therefore, polar TC-LIF is exactly equivalent to the effective symmetric $A$-bit discrete activation set with one dropped extreme level, while reducing the number of timesteps from $L-1$ to $\log_2(L)-1$.
\end{proposition}

\begin{proof}
For an $A$-bit discrete activation space, the total number of quantization levels is
\begin{equation}
L=2^A.
\end{equation}
Under the polar TC-LIF formulation, the effective support after dropping the unreachable extreme level $-L/2$ is
\begin{equation}
\mathbf{S}^l \in \left\{-\frac{L}{2}+1,\,-\frac{L}{2}+2,\ldots,\frac{L}{2}-1\right\}.
\end{equation}
Hence, the maximum magnitude is
\begin{equation}
|\mathbf{S}^l|_{\max}=\frac{L}{2}-1=2^{A-1}-1.
\end{equation}
Since any integer in the range $\{0,1,\ldots,2^{A-1}-1\}$ admits an exact binary expansion using $A-1$ bits, the magnitude $|\mathbf{S}^l|$ can be written as
\begin{equation}
|\mathbf{S}^l|=\sum_{t=1}^{A-1}2^{t-1}b_t,
\qquad b_t\in\{0,1\}.
\end{equation}
Let the sign of $\mathbf{S}^l$ be carried by spike polarity. Then we define
\begin{equation}
\hat{\mathbf{S}}^l[t]=
\begin{cases}
b_t, & \mathbf{S}^l \ge 0,\\
-b_t, & \mathbf{S}^l < 0,
\end{cases}
\qquad t=1,\ldots,A-1.
\end{equation}
It follows directly that
\begin{equation}
\mathbf{S}^l=\sum_{t=1}^{A-1}2^{t-1}\hat{\mathbf{S}}^l[t].
\end{equation}
Therefore, every value in the effective symmetric discrete support can be represented exactly using
\begin{equation}
T=A-1=\log_2(L)-1
\end{equation}
timesteps.

Next, consider the linear transform at the next layer. By linearity of matrix multiplication,
\begin{equation}
\mathbf{X}^{l+1}
=
\mathbf{W}^{l+1}\mathbf{S}^l
=
\mathbf{W}^{l+1}\sum_{t=1}^{T}2^{t-1}\hat{\mathbf{S}}^l[t]
=
\sum_{t=1}^{T}2^{t-1}\mathbf{W}^{l+1}\hat{\mathbf{S}}^l[t].
\end{equation}
Thus, the linear-layer output produced by polar TC-LIF is exactly the same as that produced by the corresponding discrete activation value under the same linear layer.

Combining the two results above, polar TC-LIF is exactly equivalent to the effective symmetric $A$-bit discrete activation set with one dropped extreme level in both numerical representation and linear-layer computation, while requiring only $T=\log_2(L)-1$ timesteps.
\end{proof}

\begin{proposition}
\textbf{(Representational Equivalence of Non-polar TC-LIF to the Non-negative $A$-bit Discrete Set).}
For an $A$-bit discrete activation space, let the number of quantization levels be $L=2^A$, and let the quantized activation at layer $l$ be denoted by
\begin{equation}
\mathbf{S}^l \in \{0,1,\ldots,L-1\}.
\end{equation}
Under the non-polar TC-LIF formulation, we set
\begin{equation}
T=\log_2(L)=A.
\end{equation}
Then every quantized activation $\mathbf{S}^l$ can be represented exactly by a weighted binary spike sequence
\begin{equation}
\mathbf{S}^l=\sum_{t=1}^{T}2^{t-1}\hat{\mathbf{S}}^l[t],
\qquad \hat{\mathbf{S}}^l[t]\in\{0,1\}.
\end{equation}
Moreover, for a linear layer with weight matrix $\mathbf{W}^{l+1}$,
\begin{equation}
\mathbf{X}^{l+1}=\mathbf{W}^{l+1}\mathbf{S}^l
=
\mathbf{W}^{l+1}\sum_{t=1}^{T}2^{t-1}\hat{\mathbf{S}}^l[t]
=
\sum_{t=1}^{T}2^{t-1}\mathbf{W}^{l+1}\hat{\mathbf{S}}^l[t].
\end{equation}
Therefore, non-polar TC-LIF is exactly equivalent to the original non-negative $A$-bit discrete activation set while reducing the number of timesteps from $L-1$ to $\log_2(L)$.
\end{proposition}

\begin{proof}
For an $A$-bit discrete activation space, the number of quantization levels is
\begin{equation}
L=2^A,
\end{equation}
and the quantized activation satisfies
\begin{equation}
\mathbf{S}^l \in \{0,1,\ldots,L-1\}.
\end{equation}
Since every integer in $\{0,1,\ldots,2^A-1\}$ has an exact binary expansion using $A$ bits, any quantized activation $\mathbf{S}^l$ can be written as
\begin{equation}
\mathbf{S}^l=\sum_{t=1}^{A}2^{t-1}\hat{\mathbf{S}}^l[t],
\qquad \hat{\mathbf{S}}^l[t]\in\{0,1\}.
\end{equation}
Therefore, every value in the non-negative $A$-bit discrete support can be represented exactly using
\begin{equation}
T=A=\log_2(L)
\end{equation}
timesteps.

Next, by linearity of matrix multiplication,
\begin{equation}
\mathbf{X}^{l+1}
=
\mathbf{W}^{l+1}\mathbf{S}^l
=
\mathbf{W}^{l+1}\sum_{t=1}^{T}2^{t-1}\hat{\mathbf{S}}^l[t]
=
\sum_{t=1}^{T}2^{t-1}\mathbf{W}^{l+1}\hat{\mathbf{S}}^l[t].
\end{equation}
Thus, the linear-layer output produced by non-polar TC-LIF is exactly the same as that produced by the original non-negative discrete activation under the same linear layer.

Combining the two results above, non-polar TC-LIF is exactly equivalent to the original non-negative $A$-bit discrete activation set in both numerical representation and linear-layer computation, while requiring only $T=\log_2(L)$ timesteps.
\end{proof}

In particular, for A4 discrete activations, polar TC-LIF requires only $T=3$ timesteps to exactly represent the effective symmetric set $\{-7,\ldots,7\}$, while non-polar TC-LIF requires $T=4$ timesteps to exactly represent the full non-negative set $\{0,\ldots,15\}$. This explains the logarithmic timestep complexity of TC-LIF compared with the linear complexity $T=L-1$ of standard unfolding.

\section{Additional Method Details}
\subsection{Integer-to-Spike and TC-LIF Algorithms}
Algorithm~\ref{alg:its} presents the complete standard integer-to-spike unfolding, where each timestep contributes a unit spike, requiring $T = L - 1$ timesteps for $A$-bit quantization. Algorithm~\ref{alg:tclif} provides a unified TC-LIF implementation: in polar mode, the quantization range is $[-L/2+1, L/2-1]$ with sign carried by spike polarity, requiring $T = A - 1$ timesteps; in non-polar mode, the quantization range is $[0, L-1]$ with all-positive activations, requiring $T = A$ timesteps. Both modes perform temporally weighted binary spike accumulation, achieving a compression ratio of $\rho = \frac{L-1}{T}$ over standard unfolding.

\subsection{Non-polar TC-LIF}
\label{app:nonpolar}
In the main text, TC-LIF adopts the polar mode, encoding sign information into spike polarity via the symmetric quantization range $[-L_m/2+1, L_m/2-1]$. Here we provide a complete description of the non-polar variant. For modality $m$, the quantization scale is defined as $s_m = \frac{\max \mathbf{U}^{l,m} - \min \mathbf{U}^{l,m}}{L_m - 1}$. The spiking process of TC-LIF quantizes the membrane potential into the asymmetric range:
\begin{equation}
\mathbf{S}^{l,m} = Q\!\left(\mathbf{U}^{l,m},\, 0,\, L_m - 1\right) = \sum_{t_m=1}^{T_m} 2^{t_m-1} \cdot \hat{\mathbf{S}}^{l,m}[t_m]
\end{equation}
Since the maximum quantized level $L_m - 1$ can be represented using $\log_2(L_m)$ bits, the value is decomposed into bit-wise components, which are encoded via temporal weighting, leading to $T_m = \log_2(L_m)$ timesteps. The temporal weighted integration follows the same form as Eq.~(13):
\begin{equation}
\mathbf{X}^{l+1,m} = \sum_{t_m=1}^{T_m} 2^{t_m-1} \cdot \left(\mathbf{W}^{l+1} \hat{\mathbf{S}}^{l,m}[t_m]\right)
\end{equation}
The non-polar variant offers a simpler implementation, but carries less information per timestep than the polar variant, requiring one additional timestep. Both modes perform temporally weighted binary spike accumulation, preserving the multiplication-free sparse addition property, and provide flexible choices for different application scenarios. Tables~\ref{tab:nonpolar_qwen} and~\ref{tab:nonpolar_intern} compare the polar and non-polar variants of TC-LIF under the same timestep budget. While both variants follow the same temporally weighted spike-accumulation scheme, the polar variant consistently outperforms the non-polar one across all benchmarks. This result suggests that representing sign information via spike polarity is beneficial for preserving model performance under aggressive timestep compression.

\begin{algorithm}[t]
\caption{Standard Integer-to-Spike}
\label{alg:its}
\begin{algorithmic}[1]
\Require Membrane potential $\mathbf{U} \in \mathbb{R}^N$, bit-width $A$, $L=2^{A}$, weight matrix $\mathbf{W}$
\Ensure Output activation $\mathbf{X}_{out}$, firing rate $\mathcal{R}$
\State $T \leftarrow L - 1$ \Comment{number of timesteps}
\State $s \leftarrow \max|\mathbf{U}| / (L-1)$ \Comment{quantization scale}
\State $\mathbf{S} \leftarrow \mathrm{clip}(\lfloor\mathbf{U}/s\rceil,\, 0,\, L-1)$ \Comment{integer spike count}
\State $\mathbf{Y}_{acc} \leftarrow \mathbf{0}$, \quad $\mathbf{R} \leftarrow \mathbf{S}$, \quad $n_{fire} \leftarrow 0$
\For{$t = 1$ \textbf{to} $T$}
    \State $\hat{\mathbf{S}}[t] \leftarrow \mathbf{1}[\mathbf{R} \geq 1]$ \Comment{fire if residual remains, $\hat{\mathbf{S}}[t]\in\{0,1\}$}
    \State $\mathbf{R} \leftarrow \mathbf{R} - \hat{\mathbf{S}}[t]$ \Comment{update residual}
    \For{each $i$ with $\hat{S}_i[t] \neq 0$}
        \State $\mathbf{Y}_{acc} \leftarrow \mathbf{Y}_{acc} + \mathbf{W}_{:,i}$ \Comment{spike fires: accumulate $\mathbf{W}_{:,i}$}
    \EndFor
    \State $n_{fire} \leftarrow n_{fire} + \|\hat{\mathbf{S}}[t]\|_0$
\EndFor
\State $\mathbf{X}_{out} \leftarrow s\,\mathbf{Y}_{acc}$
\State $\mathcal{R} \leftarrow n_{fire} / (T \cdot N)$ \Comment{average firing rate}
\State \Return $\mathbf{X}_{out}$, $\mathcal{R}$
\end{algorithmic}
\end{algorithm}

\begin{algorithm}[t]
\caption{TC-LIF Inference (Non-polar and Polar)}
\label{alg:tclif}
\begin{algorithmic}[1]
\Require Membrane potential $\mathbf{U}\in\mathbb{R}^{N}$, bit-width $A$, $L=2^{A}$, weight matrix $\mathbf{W}\in\mathbb{R}^{M\times N}$, polarity flag \textsc{polar}
\Ensure Output activation $\mathbf{X}_{out}$, firing rate $\mathcal{R}$, compression ratio $\rho$
\If{\textsc{polar}}
    \State $T \leftarrow A-1$ \Comment{$T = \log_2(L)-1$}
    \State $s \leftarrow \max|\mathbf{U}| / (L/2-1)$
    \State $\mathbf{S}_{int} \leftarrow \mathrm{clip}(\lfloor\mathbf{U}/s\rceil,\,-L/2+1,\,L/2-1)$
    \State $\boldsymbol{\sigma} \leftarrow \mathrm{sign}(\mathbf{S}_{int})$, \quad $\mathbf{V} \leftarrow |\mathbf{S}_{int}|$ \Comment{$\hat{\mathbf{S}}[t]\in\{0,1\}$, sign carried by $\boldsymbol{\sigma}$}
\Else
    \State $T \leftarrow A$ \Comment{$T = \log_2(L)$}
    \State $s \leftarrow \max|\mathbf{U}| / (L-1)$
    \State $\mathbf{S}_{int} \leftarrow \mathrm{clip}(\lfloor\mathbf{U}/s\rceil,\,0,\,L-1)$
    \State $\boldsymbol{\sigma} \leftarrow \mathbf{1}$, \quad $\mathbf{V} \leftarrow \mathbf{S}_{int}$
\EndIf
\State $\mathbf{Y}_{acc} \leftarrow \mathbf{0}\in\mathbb{R}^{M}$, \quad $\mathbf{R} \leftarrow \mathbf{V}$, \quad $n_{fire} \leftarrow 0$
\For{$t = T$ \textbf{downto} $1$}
    \State $w_t \leftarrow 2^{t-1}$
    \State $\hat{\mathbf{S}}[t] \leftarrow \mathbf{1}[\mathbf{R} \geq w_t]$
    \State $\mathbf{R} \leftarrow \mathbf{R} - w_t\,\hat{\mathbf{S}}[t]$
    \For{each $i$ with $\hat{S}_i[t] \neq 0$}
        \State \textbf{if} $\sigma_i < 0$: $\mathbf{Y}_{acc} \leftarrow \mathbf{Y}_{acc} - w_t\,\mathbf{W}_{:,i}$ \Comment{inhibitory spike}
        \State \textbf{else}: \phantom{$\sigma_i < 0$: }$\mathbf{Y}_{acc} \leftarrow \mathbf{Y}_{acc} + w_t\,\mathbf{W}_{:,i}$ \Comment{excitatory spike}
    \EndFor
    \State $n_{fire} \leftarrow n_{fire} + \|\hat{\mathbf{S}}[t]\|_0$
\EndFor
\State $\mathbf{X}_{out} \leftarrow s\,\mathbf{Y}_{acc}$ \Comment{sign absorbed in loop}
\State $\mathcal{R} \leftarrow n_{fire} / (T \cdot N)$
\State $\rho \leftarrow (L-1) / T$ \Comment{compression ratio over standard unfolding}
\State \Return $\mathbf{X}_{out},\,\mathcal{R},\,\rho$
\end{algorithmic}
\end{algorithm}

\begin{table}[t]
\centering
\caption{Comparison of polar and non-polar TC-LIF variants for SpikeMLLM(QuaRot+MSTS+TC-LIF) at $T_v/T_t=3/4$ on Qwen2VL-7B.}
\label{tab:nonpolar_qwen}
\begin{adjustbox}{width=\linewidth} 
\begin{tabular}{lccccc}
\toprule
TC-LIF & OCRBench & MME & TextVQA & DocVQA & ScienceQA \\
\midrule
FP16 & 835 & 2285 & 84.11 & 93.95 & 85.67 \\
Non-polar & 768 & 2087 & 81.42 & 89.64 & 82.05 \\
Polar & \textbf{823} & \textbf{2236} & \textbf{83.46} & \textbf{93.31} & \textbf{85.57} \\
\bottomrule
\end{tabular}
\end{adjustbox}
\end{table}

\begin{table}[t]
\centering
\caption{Comparison of polar and non-polar TC-LIF variants for SpikeMLLM(QuaRot+MSTS+TC-LIF) at $T_v/T_t=3/4$ on InternVL2-8B.}
\label{tab:nonpolar_intern}
\begin{adjustbox}{width=\linewidth} 
\begin{tabular}{lccccc}
\toprule
TC-LIF & OCRBench & MME & TextVQA & DocVQA & ScienceQA \\
\midrule
FP16 & 795 & 2211 & 77.71 & 91.05 & 97.07 \\
Non-polar & 751 & 2082 & 75.00 & 85.78 & 95.88 \\
Polar & \textbf{791} & \textbf{2198} & \textbf{76.98} & \textbf{89.85} & \textbf{96.83} \\
\bottomrule
\end{tabular}
\end{adjustbox}
\end{table}

\begin{table}[t]
\centering
\caption{Firing rate and FLOPs analysis of SpikeMLLM on InternVL2-8B. }
\label{tab:efficiency_intern}
\begin{adjustbox}{width=\linewidth} 
\begin{tabular}{lccc}
\toprule
Method & Time Step & Firing Rate & FLOPs(T) \\
\midrule
FP16 & N/A & N/A & 9.28 \\
SpikeMLLM(RTN) & 255 & 0.50 & 18.57 \\
SpikeMLLM(GPTQ) & 255 & 0.50 & 18.57 \\
SpikeMLLM(QuaRot) & 7 & 0.57 & 0.58 \\
SpikeMLLM(QuaRot+MSTS+TC-LIF) & 2/3 & 0.30 & 0.17 \\
W4A4 & N/A & N/A & 0.58 \\
SpikeMLLM(QuaRot) & 15 & 0.53 & 1.16 \\
SpikeMLLM(QuaRot+MSTS+TC-LIF) & 3/4 & 0.30 & 0.27 \\
\bottomrule
\end{tabular}
\end{adjustbox}
\end{table}

\section{Experimental and Computational Details}
\subsection{Sensitivity Analysis on Weight Bit-width}

We analyze the sensitivity of SpikeMLLM(QuaRot+MSTS+TC-LIF) to weight bit-width, including LLM weight bit-width (W4/W6/W8) and ViT weight bit-width (W4/W6/W8), with results shown in Tables~\ref{tab:llm_bitwidth} and~\ref{tab:vit_bitwidth}, respectively.

For LLM weights, performance remains stable across W4 to W8 on both Qwen2VL-7B and InternVL2-8B, indicating that SpikeMLLM is insensitive to LLM weight bit-width. For ViT weights, W6 achieves the best overall performance, while W4 shows a slight degradation and W8 is comparable to W6, with all configurations remaining within a narrow performance range.

These results demonstrate that SpikeMLLM maintains stable performance across different weight bit-width configurations, validating the robustness of the framework to weight quantization.

\subsection{FLOPs Computation}
\label{appendix:flops}
We refer to the $q$-bit operation FLOPs calculation method from Q-DETR~\cite{xu2023qdetr}. The FLOPs for 2-bit operations are $\frac{1}{32}$ of 32-bit FLOPs, for 3-bit operations $\frac{1}{16}$, and for 4-bit operations $\frac{1}{8}$~\cite{xu2023qdetr, liu2020birealnet}, since current hardware can parallelize bitwise operations via XNOR and popcount. Higher-bit operations can be decomposed into multiple lower-bit operations, and we compute $q$-bit FLOPs following this principle for all methods in this paper. For spiking neural networks, non-polar spikes $\hat{\mathbf{S}} \in \{0, 1\}$ correspond to 1-bit operations; polar spikes $\hat{\mathbf{S}} \in \{0, 1\}$ or $\hat{\mathbf{S}} \in \{-1, 0\}$ correspond to 2-bit operations. The FLOPs for spike-based methods are thus computed as:
\begin{equation}
    \text{FLOPs}_\text{spike} = \text{FLOPs}_{q\text{-bit}} \times T \times R
\end{equation}
where $T$ denotes the number of timesteps and $R$ denotes the average firing rate~\cite{xu2025neurormorphic}.

\subsection{Qualitative Examples}
Figure~\ref{fig:qualitative} presents qualitative examples comparing SpikeMLLM (QuaRot+MSTS+TC-LIF, $T_v/T_t=3/4$) with the FP16 baseline on representative multimodal inputs.

\section{Hardware Implementation Details}

\subsection{Hardware Co-design Accelerator Architecture}
\label{hardwaredetail}

\begin{table}[t]
\centering
\caption{Hardware Implementation Metrics}
\label{ourhardware}
\begin{adjustbox}{width=1\linewidth} 
\begin{tabular}{lc}
\toprule
Metric & Value \\
\midrule
Technology [nm] & 28 \\
Voltage [V] & 0.8 \\
Frequency [MHz] & 333 \\
Critical Path Delay [ns] & 2.04 \\
Power [mW] & 484 \\
Die Area [mm$^2$] & 76.27 \\
On-chip Memory [MB] & 44 \\
Peak Throughput [TOPS] & 5.46 \\
Peak Area Efficiency [TOPS/mm$^2$] & 0.072 \\
Peak Energy Efficiency [TOPS/W] & 11.28 ($S=50\%$) \\
\bottomrule
\end{tabular}
\end{adjustbox}
\end{table}

Supplementary Fig.~\ref{fig:toparchw16} illustrates the overall organization of the proposed architecture, which consists of four major components: Memory Banks, Bank Select Unit, Control Unit, and Spiking Mat-mul Unit. Rather than being a generic matrix-multiplication engine, the architecture is specifically co-optimized for the spike-based computation pattern, multiplier-free execution path, and high-reuse dataflow of SpikeMLLM. 

The Memory Banks store input activations, weights, and intermediate/output results, and provide parallel data supply through a banked organization. The Bank Select Unit is placed between the memory subsystem and the compute fabric, where it selects the target banks according to the current execution phase and maps the fetched data to the corresponding compute channels. The Control Unit generates the control signals for memory access, bank selection, and array execution, thereby coordinating data loading, computation progression, and result write-back. The Spiking Mat-mul Unit serves as the core compute engine and is organized as a regular PE array for executing spike-driven matrix multiplication.

From a dataflow perspective, the architecture adopts a broadcast-based operand reuse scheme. In each compute phase, the on-chip memory subsystem only needs to supply 2T input operands to sustain the array: one group of T operands is broadcast along one array dimension, while the other group of T operands is broadcast along the orthogonal dimension. This allows T×T PEs to update partial results in parallel. Compared to feeding each PE independently, this design can reduce the required data-delivery bandwidth and improves the reuse of spatial operands within the array. In the implementation of this work, the Spiking Mat-mul Unit is organized as a 16×16 regular array, corresponding to 256 parallel PEs as shown in Supplementary Fig.~\ref{fig:toparchw16}.

\subsection{RTL Synthesis Details}
\label{RTL}
The proposed accelerator was synthesized using the SMIC 28 nm CMOS technology library at 0.8 V and 25~$^\circ$C, targeting a clock frequency of 333 MHz, corresponding to a 3 ns clock period. Timing analysis was conducted across multiple process corners, including SS, TT, and FF, to verify timing closure. For power evaluation, post-synthesis switching activity was captured in VCD format and back-annotated into PrimeTime for power estimation.All hardware results are reported under this synthesis configuration.

As summarized in Supplementary Table~\ref{ourhardware}, the synthesized design occupies 76.27 mm$^2$ and integrates 44 MB of on-chip memory. The post-synthesis critical path delay is 2.04 ns, which satisfies the target 3 ns timing constraint with sufficient timing margin. Under this implementation, the accelerator achieves a peak throughput of 5.46 TOPS. Here, the reported peak throughput refers to the theoretical maximum compute capability obtained under the assumption that all processing elements are fully utilized with ideal data supply, without accounting for performance loss caused by memory access stalls, pipeline bubbles, or control overhead. Based on this peak-performance assumption, the corresponding peak area efficiency and peak energy efficiency are 0.072 TOPS/mm$^2$ and 11.28 TOPS/W at $S=50\%$, respectively. 

\clearpage

\begin{table*}[!h]
\centering
\caption{Ablation on LLM weight bit-width for SpikeMLLM(QuaRot+MSTS+TC-LIF) at $T_v/T_t=3/4$ (ViT W6 fixed).}
\label{tab:llm_bitwidth}
\begin{adjustbox}{width=1\linewidth} 
\begin{tabular}{llccccccccc}
\toprule
Model & Method & ViT W & LLM W & Spiking & Time Step & OCRBench & MME & TextVQA & DocVQA & ScienceQA \\
\midrule
\multirow{4}{*}{Qwen2VL-7B}
& FP16 & -- & -- & $\times$ & N/A & 835 & 2285 & 84.11 & 93.95 & 85.67 \\
& SpikeMLLM(QuaRot+MSTS+TC-LIF) & 6 & 4 & $\checkmark$ & 3/4 & 823 & 2236 & 83.46 & 93.31 & 85.57 \\
& SpikeMLLM(QuaRot+MSTS+TC-LIF) & 6 & 6 & $\checkmark$ & 3/4 & 825 & 2258 & 83.76 & 93.54 & 85.18 \\
& SpikeMLLM(QuaRot+MSTS+TC-LIF) & 6 & 8 & $\checkmark$ & 3/4 & 815 & 2247 & 83.61 & 93.21 & 85.72 \\
\midrule
\multirow{4}{*}{InternVL2-8B}
& FP16 & -- & -- & $\times$ & N/A & 795 & 2211 & 77.71 & 91.05 & 97.07 \\
& SpikeMLLM(QuaRot+MSTS+TC-LIF) & 6 & 4 & $\checkmark$ & 3/4 & 791 & 2198 & 76.98 & 89.85 & 96.83 \\
& SpikeMLLM(QuaRot+MSTS+TC-LIF) & 6 & 6 & $\checkmark$ & 3/4 & 786 & 2210 & 76.78 & 90.04 & 96.58 \\
& SpikeMLLM(QuaRot+MSTS+TC-LIF) & 6 & 8 & $\checkmark$ & 3/4 & 798 & 2226 & 76.79 & 90.40 & 96.83 \\
\bottomrule
\end{tabular}
\end{adjustbox}

\vspace{24pt}

\centering
\caption{Ablation on ViT weight bit-width for SpikeMLLM(QuaRot+MSTS+TC-LIF) at $T_v/T_t=3/4$ (LLM W4 fixed).}
\label{tab:vit_bitwidth}
\begin{adjustbox}{width=1\linewidth} 
\begin{tabular}{llccccccccc}
\toprule
Model & Method & ViT W & LLM W & Spiking & Time Step & OCRBench & MME & TextVQA & DocVQA & ScienceQA \\
\midrule
\multirow{4}{*}{Qwen2VL-7B}
& FP16 & -- & -- & $\times$ & N/A & 835 & 2285 & 84.11 & 93.95 & 85.67 \\
& SpikeMLLM(QuaRot+MSTS+TC-LIF) & 4 & 4 & $\checkmark$ & 3/4 & 819 & 2188 & 83.27 & 93.24 & 85.23 \\
& SpikeMLLM(QuaRot+MSTS+TC-LIF) & 6 & 4 & $\checkmark$ & 3/4 & 823 & 2236 & 83.46 & 93.31 & 85.57 \\
& SpikeMLLM(QuaRot+MSTS+TC-LIF) & 8 & 4 & $\checkmark$ & 3/4 & 816 & 2257 & 83.52 & 93.18 & 85.52 \\
\midrule
\multirow{4}{*}{InternVL2-8B}
& FP16 & -- & -- & $\times$ & N/A & 795 & 2211 & 77.71 & 91.05 & 97.07 \\
& SpikeMLLM(QuaRot+MSTS+TC-LIF) & 4 & 4 & $\checkmark$ & 3/4 & 791 & 2198 & 76.98 & 89.85 & 96.83 \\
& SpikeMLLM(QuaRot+MSTS+TC-LIF) & 6 & 4 & $\checkmark$ & 3/4 & 791 & 2198 & 76.98 & 89.85 & 96.83 \\
& SpikeMLLM(QuaRot+MSTS+TC-LIF) & 8 & 4 & $\checkmark$ & 3/4 & 785 & 2177 & 77.01 & 89.74 & 96.93 \\
\bottomrule
\end{tabular}
\end{adjustbox}

\vspace{24pt}

\centering
\caption{Ablation of MED-guided layer-wise timestep allocation on Qwen2VL-7B and InternVL2-8B, based on SpikeMLLM(QuaRot+MSTS+TC-LIF). MSTS-reverse denotes assigning higher timesteps to layers with lower MED under the same mixed timestep budget. The proposed MED-guided allocation consistently achieves better performance on these representative models.}
\label{tab:reverse}
\begin{adjustbox}{width=1\linewidth} 
\begin{tabular}{llcccccccc}
\toprule
Model & Method & Spiking & Time Step & OCRBench & MME & TextVQA & DocVQA & ScienceQA & Avg gap \\
\midrule
& FP16 & $\times$ & N/A & 835 & 2285 & 84.11 & 93.95 & 85.67 & — \\
& MSTS & \checkmark & 2.1/3.1 & 754 & 2070 & 78.99 & 86.49 & 80.17 & 7.91 \\
& MSTS-reverse & \checkmark & 2.1/3.1 & 717 & 2029 & 78.24 & 84.85 & 81.61 & 9.34 \\
\multirow{-4}{*}{Qwen2VL-7B}
& MSTS & \checkmark & 2.2/3.2 & 751 & 2068 & 79.63 & 87.52 & 80.66 & 7.51 \\
& MSTS-reverse & \checkmark & 2.2/3.2 & 742 & 1990 & 79.78 & 85.59 & 80.17 & 8.90 \\
& MSTS & \checkmark & 2.3/3.3 & 765 & 2109 & 80.97 & 89.12 & 81.95 & 5.86 \\
& MSTS-reverse & \checkmark & 2.3/3.3 & 752 & 2083 & 80.07 & 87.56 & 81.71 & 7.00 \\
\midrule
& FP16 & $\times$ & N/A & 795 &	2211 & 77.71& 91.05 & 97.07 & — \\
& MSTS & \checkmark & 2.1/3.1 & 740 & 2081 & 72.70 & 83.16 & 93.60 & 6.30  \\
& MSTS-reverse & \checkmark & 2.1/3.1 & 725 & 1983 & 71.60 & 80.19 & 93.65 & 8.49  \\
\multirow{-4}{*}{InternVL2-8B}
& MSTS & \checkmark & 2.2/3.2 & 741 & 2083 & 73.51 & 85.41 & 94.55 & 5.36  \\
& MSTS-reverse & \checkmark & 2.2/3.2 & 729 & 1999 & 72.10 & 80.89 & 93.65 & 7.96  \\
& MSTS & \checkmark & 2.3/3.3 & 741 & 2117 & 74.39 & 85.66 & 94.99 & 4.67  \\
& MSTS-reverse & \checkmark & 2.3/3.3 & 723 & 1980 & 73.63 & 81.49 & 93.36 & 7.82  \\
\bottomrule
\end{tabular}
\end{adjustbox}
\end{table*}

\newpage

\begin{figure*}[p]
\vspace*{\fill}
  \centering
  \includegraphics[width=0.95\linewidth]{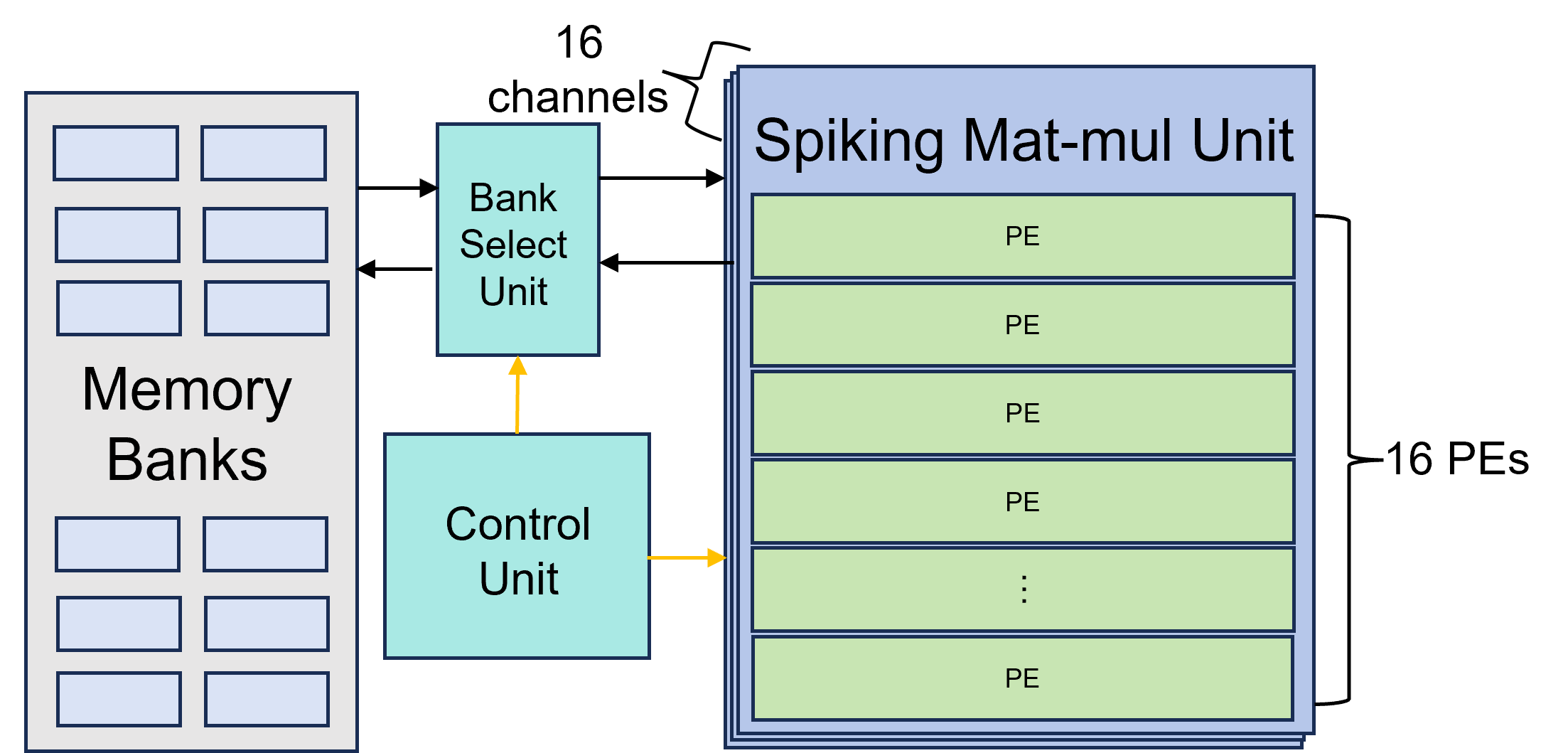}
  \caption{Overall Architecture of Co-design Accelerator}
  \label{fig:toparchw16}
  \vspace*{\fill}
\end{figure*}

\newpage

\begin{figure*}[h]
  \centering
  \begin{subfigure}[b]{0.48\linewidth}
    \centering
    \includegraphics[width=\linewidth]{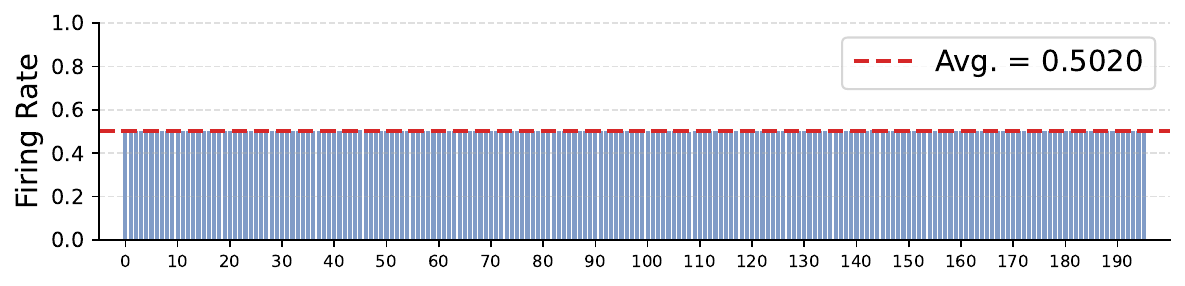}
    \caption{SpikeMLLM(RTN) ($T=255$)}
  \end{subfigure}
  \hfill
  \begin{subfigure}[b]{0.48\linewidth}
    \centering
    \includegraphics[width=\linewidth]{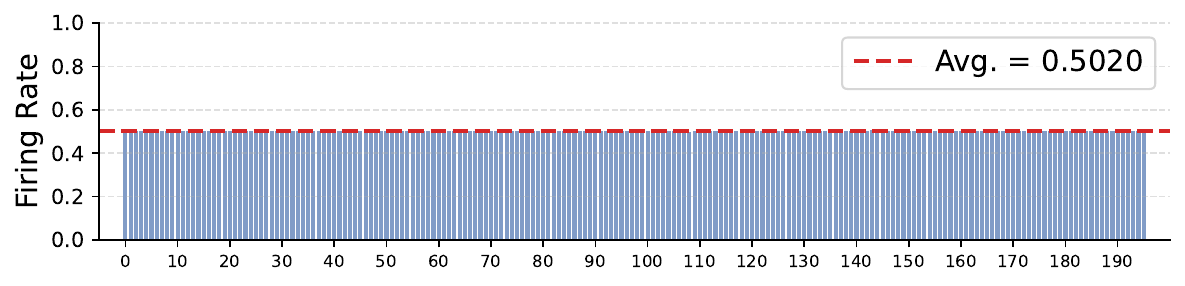}
    \caption{SpikeMLLM(GPTQ) ($T=255$)}
  \end{subfigure}

  \vspace{0.5em}
  \begin{subfigure}[b]{0.48\linewidth}
    \centering
    \includegraphics[width=\linewidth]{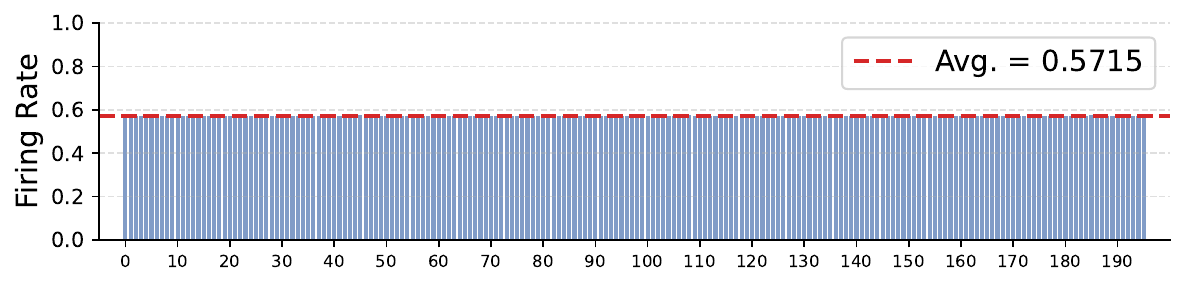}
    \caption{SpikeMLLM(QuaRot) ($T=7$)}
  \end{subfigure}
  \hfill
  \begin{subfigure}[b]{0.48\linewidth}
    \centering
    \includegraphics[width=\linewidth]{figs/qwen2vl_fr_per_layer_rot_15.pdf}
    \caption{SpikeMLLM(QuaRot) ($T=15$)}
  \end{subfigure}

  \vspace{0.5em}
  \begin{subfigure}[b]{0.48\linewidth}
    \centering
    \includegraphics[width=\linewidth]{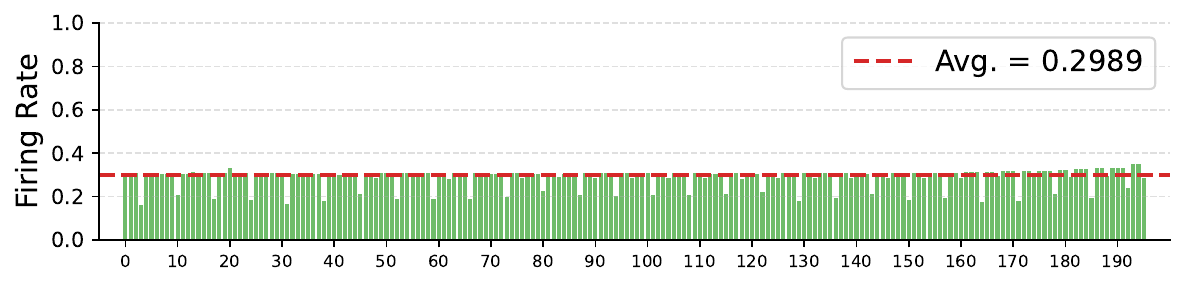}
    \caption{SpikeMLLM(QuaRot+MSTS+TC-LIF) ($T_v/T_t=2/3$)}
  \end{subfigure}
  \hfill
  \begin{subfigure}[b]{0.48\linewidth}
    \centering
    \includegraphics[width=\linewidth]{figs/qwen2vl_fr_per_layer_our34.pdf}
    \caption{SpikeMLLM(QuaRot+MSTS+TC-LIF) ($T_v/T_t=3/4$)}
  \end{subfigure}

  \caption{Per-layer firing rate distribution of different methods on Qwen2VL-7B.}
  \label{fig:firing_rate_full_qwen} 

\vspace{24pt}

  \centering
  \begin{subfigure}[b]{0.48\linewidth}
    \centering
    \includegraphics[width=\linewidth]{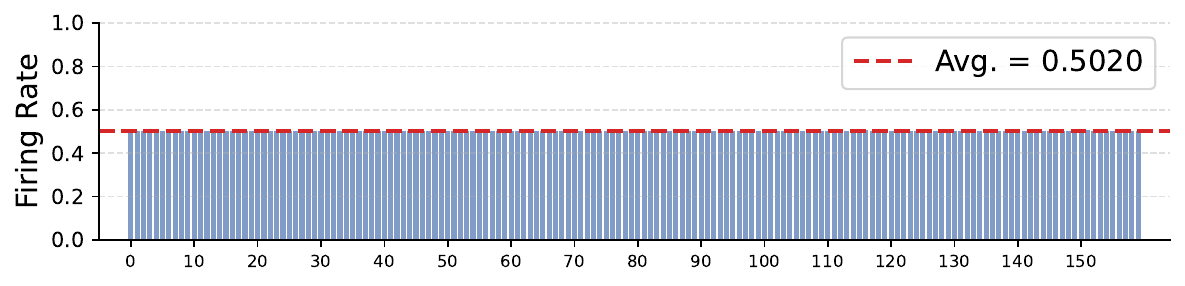}
    \caption{SpikeMLLM(RTN) ($T=255$)}
  \end{subfigure}
  \hfill
  \begin{subfigure}[b]{0.48\linewidth}
    \centering
    \includegraphics[width=\linewidth]{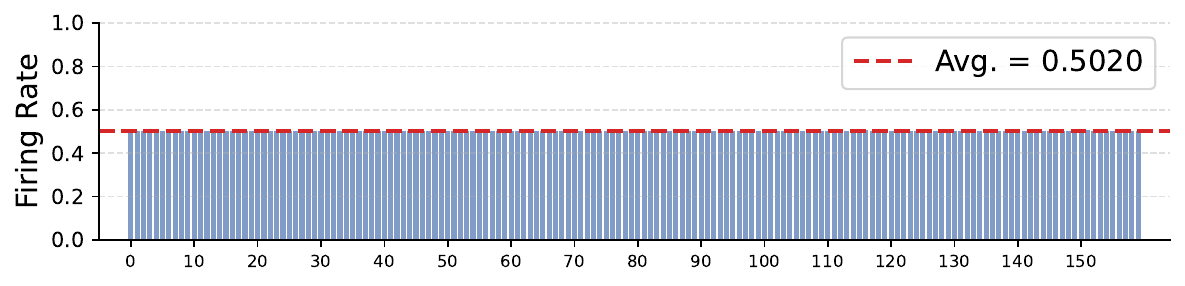}
    \caption{SpikeMLLM(GPTQ) ($T=255$)}
  \end{subfigure}

  \vspace{0.5em}
  \begin{subfigure}[b]{0.48\linewidth}
    \centering
    \includegraphics[width=\linewidth]{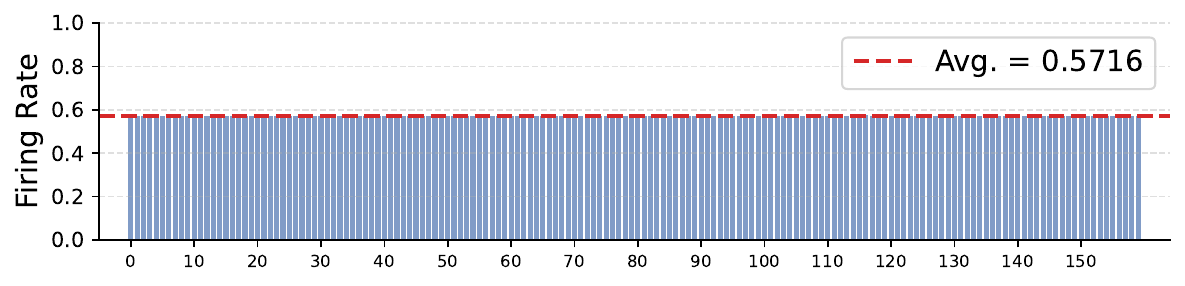}
    \caption{SpikeMLLM(QuaRot) ($T=7$)}

  \end{subfigure}
  \hfill
  \begin{subfigure}[b]{0.48\linewidth}
    \centering
    \includegraphics[width=\linewidth]{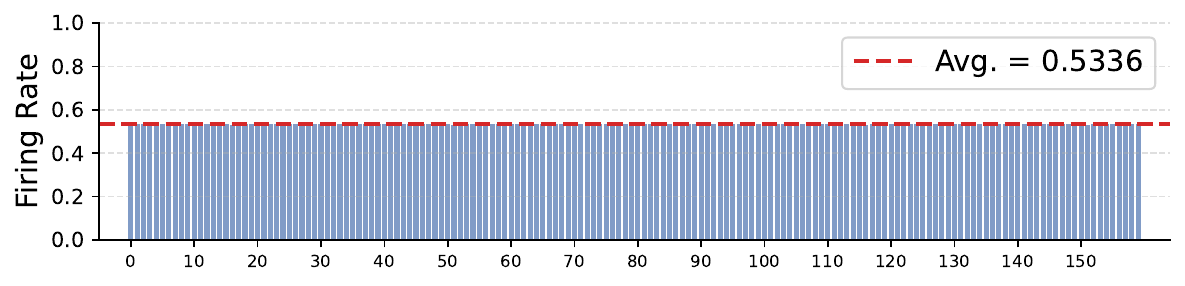}
    \caption{SpikeMLLM(QuaRot) ($T=15$)}
  \end{subfigure}

  \vspace{0.5em}
  \begin{subfigure}[b]{0.48\linewidth}
    \centering
    \includegraphics[width=\linewidth]{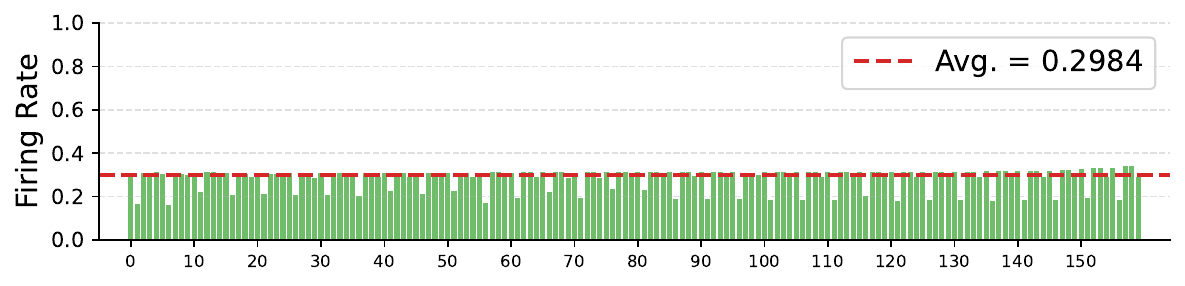}
    \caption{SpikeMLLM(QuaRot+MSTS+TC-LIF) ($T_v/T_t=2/3$)}
  \end{subfigure}
  \hfill
  \begin{subfigure}[b]{0.48\linewidth}
    \centering
    \includegraphics[width=\linewidth]{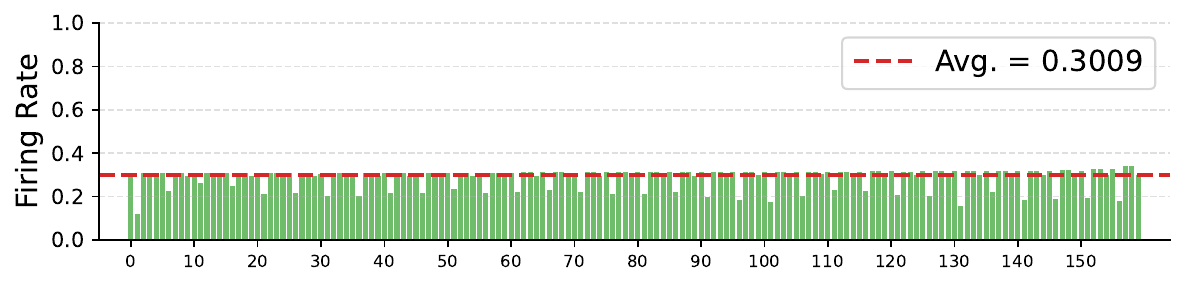}
    \caption{SpikeMLLM(QuaRot+MSTS+TC-LIF) ($T_v/T_t=3/4$)}
  \end{subfigure}

  \caption{Per-layer firing rate distribution of different methods on InternVL2-8B.}
  \label{fig:firing_rate_full_intern}
\end{figure*}

\newpage

\begin{figure*}[p]
\vspace*{\fill}
  \centering
  \includegraphics[width=0.99\linewidth]{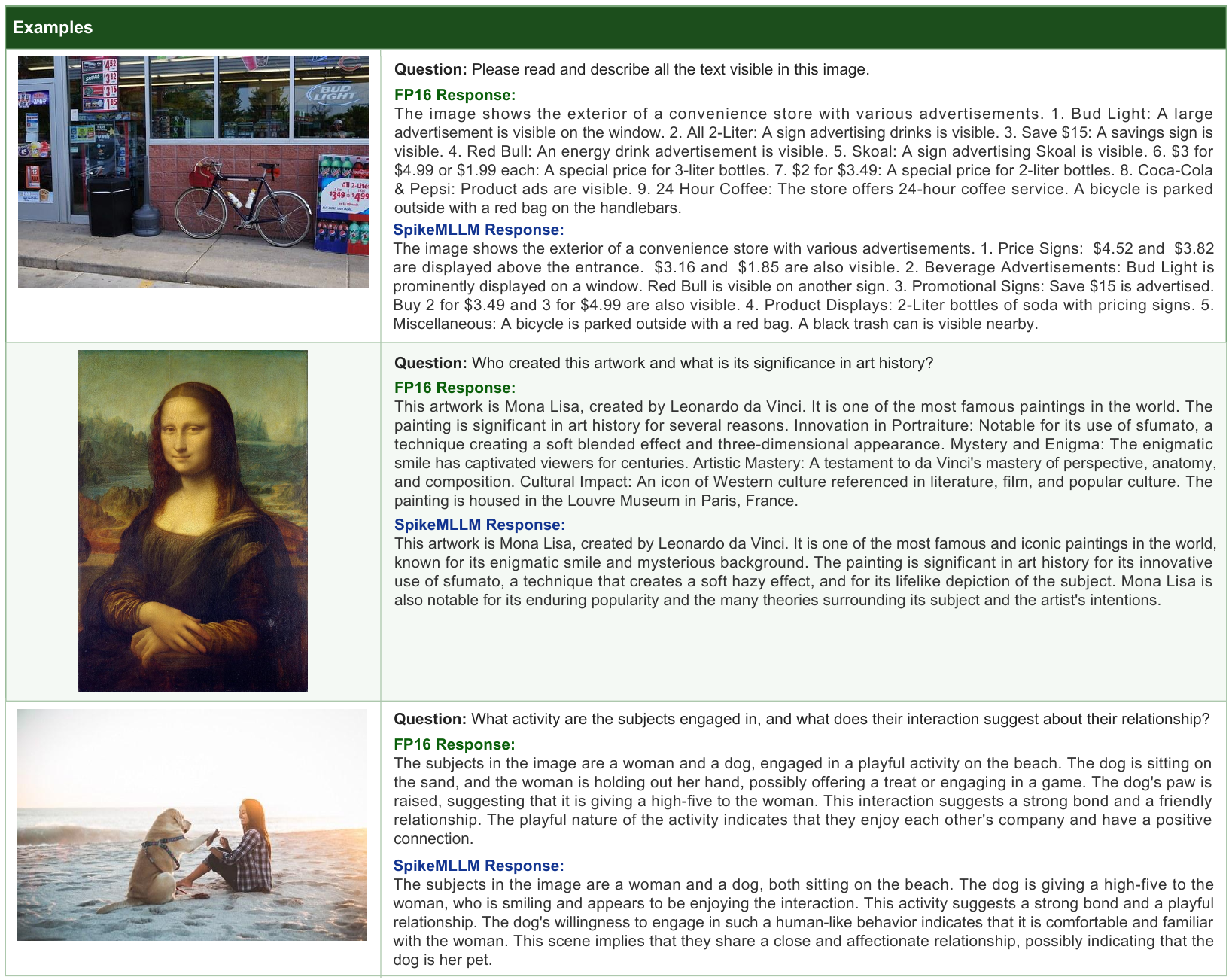}
  \caption{Example responses generated by SpikeMLLM(QuaRot+MSTS+TC-LIF, $T_v/T_t=3/4$) and the FP16 baseline on representative multimodal inputs. The responses remain highly consistent under aggressive timestep compression.}
  \label{fig:qualitative}
  \vspace*{\fill}
\end{figure*}

\end{document}